\documentclass{article}
\usepackage[numbers]{natbib}
\usepackage{microtype}
\usepackage{graphicx}
\usepackage{hyperref}
\usepackage{breakurl}
\usepackage{url}
\usepackage{booktabs} 
\usepackage{hyperref}
\usepackage{amsmath}
\usepackage{xcolor}
\usepackage{subcaption}
\usepackage{chngcntr}
\usepackage[normalem]{ulem}
\usepackage{amsfonts}
\usepackage[accepted]{icml2020}
\usepackage{amsthm}
\newtheorem{prop}{Proposition}
\newtheorem{cor}{Corollary}

\newcommand{\ptrain}{p_{\text{train}}}
\newcommand{\ptest}{p_{\text{test}}}
\newcommand{\pjoint}{\mathcal{P}^{\adro}_{\x,\y}}
\newcommand{\pcov}{\mathcal{P}^{\adro}_{\x}}
\newcommand{\pconf}{\mathcal{P}^{\adro}_{\x,\conf}}

\newcommand{\ppot}{\mathcal{P}}

\newcommand{\x}{x}
\newcommand{\y}{y}
\newcommand{\conf}{c}  
\newcommand{\atrue}{\alpha^*}
\newcommand{\adro}{\alpha}

\newcommand{\defeq}{:=}
\newcommand{\confcondrisk}{\E[\loss(\param; (\x,\y)) \mid \x, \conf]}

\newcommand{\mc}{\mathcal}

\newcommand{\param}{\theta}
\newcommand{\loss}{\ell}
\newcommand{\E}{\mathbb{E}}
\newcommand{\EE}{\mathbb{E}}
\newcommand{\hingeBig}[1]{\Big({#1}\Big)_+}
\newcommand{\hinge}[1]{({#1})_+}
\newcommand{\norm}[1]{\left\|{#1}\right\|} 
\newcommand{\ol}{\overline}
\newcommand{\olmc}[1]{\overline{\mathcal{#1}}}

\icmltitlerunning{Robustness to Spurious Correlations via Human Annotations}
\begin{document}
\twocolumn[
\icmltitle{Robustness to Spurious Correlations via Human Annotations}

\begin{icmlauthorlist}
\icmlauthor{Megha Srivastava}{to}
\icmlauthor{Tatsunori Hashimoto}{to}
\icmlauthor{Percy Liang}{to}
\end{icmlauthorlist}

\icmlaffiliation{to}{Computer Science Department, Stanford University}
\icmlcorrespondingauthor{Megha Srivastava}{megha@cs.stanford.edu}
\icmlkeywords{Machine Learning, ICML}

\vskip 0.3in
]

\printAffiliationsAndNotice{} 

\begin{abstract}
The reliability of machine learning systems critically assumes that the associations between features and labels remain similar between training and test distributions. However, \textit{unmeasured variables}, such as confounders, break this assumption---useful correlations between features and labels at training time can become useless or even harmful at test time. For example, high obesity is generally predictive for heart disease, but this relation may not hold for smokers who generally have lower rates of obesity and higher rates of heart disease. We present a framework for making models robust to spurious correlations by leveraging humans' common sense knowledge of causality. Specifically, we use human annotation to augment each training example with a potential unmeasured variable (i.e. an underweight patient with heart disease may be a smoker), reducing the problem to a covariate shift problem. We then introduce a new distributionally robust optimization objective over unmeasured variables (UV-DRO) to control the worst-case loss over possible test-time shifts. Empirically, we show improvements of 5--10\% on a digit recognition task confounded by rotation, and 1.5--5\% on the task of analyzing NYPD Police Stops confounded by location.

\end{abstract}

\section{Introduction} \label{intro}
The increasing use of machine learning in socioeconomic problems as well as high-stakes decision-making emphasizes the importance of designing models that can perform well over a wide range of users and conditions \citep{barocas2016,blodgett2016,hovy2015,tatman2017}. In some cases, the set of target users and test distributions are known---for example, research in the fair machine learning community has largely been motivated by case studies such as face-recognition systems performing poorly on populations with dark skin color \citep{buolamwini2018gender}. However, there exist many more distributional shifts that the designer of a machine learning system may have been unaware of when collecting data, or may be impossible to measure. Existing approaches such as distributional robustness \cite{bental2013robust, lam2015quantifying} and domain adaptation \cite{mansour2009domain, blitzer2011domain,  gong2013reshaping} require either a priori specifying the distribution shifts, or sampling from the target test distributions. How can we ensure that a model performs reliably at test time without explicitly specifying the domain shifts? 

Existing research on human-in-the-loop systems and crowdsourcing have shown that humans have a rich understanding of the plausible domain shifts in our world, as well as how these changes affect the prediction task \citep{talmor2019commonsenseqa, mostafazadeh2016corpus}. Can we leverage humans' strong prior knowledge to understand the possible distribution shifts for a specific machine learning task? The key idea of our paper is to  use human commonsense reasoning as a source of information about potential test-time shifts, and effectively use this information to learn robust models.  

To see how human annotations may help, consider the task of creating large-scale diagnostic models for medicine. Although these models are trained on large amounts of data, they almost invariably lack features for important risk factors that were either hidden for legal reasons (e.g. health insurance providers cannot collect genetic information), privacy concerns (e.g. collection of ethnic information \citep{michele2004eliminating}), or simply unobserved (e.g. drug use). For example, a diagnostic model for heart disease trained on the general population may learn to predict heart disease based upon obesity. However, when used in a drug rehabilitation facility with former smokers, this model may perform poorly, as smokers are often underweight and have high heart disease risk \citep{jarvik1991nicotine}. In this case, smoking is an unmeasured confounder which degrades the model's robustness. A human expert could help with such confounded shifts by annotating the data and identifying that examples with low obesity but significant heart problems may be due to smoking. We would then be able to train our model to be robust to distribution shifts over these unmeasured variables (i.e. if our test set consists primarily of smokers). 

The setting we consider is a prediction task where given features ($\x$) and labels ($\y$) from a training distribution, our goal is to perform well at predicting $y$ given $x$ on an a priori unknown test distribution. To make this task possible, we hypothesize that the distribution shift occurs solely over the features $x$ and a set of unmeasured variables $c$---which can encode obvious confounding factors such as the location of the collected data, as well as more complex factors such as time or demographic information of individuals. Although this assumption reduces the problem to the well-studied covariate shift case, we cannot apply any of these algorithms directly as $c$ is unobserved.

The key insight of our work is that if we design our model to only depend on the features $x$ (and not the unmeasured variables $\conf$), we \emph{do not} need to recover the true value of $\conf$. Instead, our procedure only requires samples from the conditional distribution $\conf \mid \x, \y$ during training. This property allows us to augment the training data with $\conf$ using crowdsourcing, and leverage human commonsense reasoning to define the potential test-time shifts over unmeasured variables. We will first augment the dataset with approximate $\ol{\conf}$ by asking  humans for natural language descriptions of additional reasons why features $x$ would lead to a label $y$. Eliciting $\conf$ in natural language means that we do not have to specify the set of potential unmeasured variables, and allows annotators to easily express a diverse and rich class of $\conf$'s. We then use these annotations as a way to learn a model that predicts labels $y$ given only the observable features $x$ under potential distribution shifts on $(\x,\conf)$.

\section{Problem Statement} \label{setup}
Formally, consider a prediction problem where we observe features $\x$, and predict a label $\y$. A model $\param$ suffers loss $\ell((\x,\y),\param)$, and we train this model using samples $(\x,\y) \sim \ptrain$. While standard practice minimizes risk with respect to the training distribution,
\begin{equation}
  \EE_{\ptrain}[\ell((\x,\y);\param)],\label{eq:erm}
\end{equation}
this approach can fail when $\ptest \neq \ptrain$, as is common in real-world tasks that involve domain adaptation. For example, the training distribution may be affected by annotation biases \cite{geva2019annotator} or underrepresentation of minority groups \cite{oren2019drolm,hashimoto2018repeated} compared to the test distribution. In this situation we would like the model to perform well over the set of potential test distributions $\mathcal{P}$ by minimizing
\begin{equation}
  \mathcal{R}(\param, \mc{P}) := \sup_{P\in\mc{P}}\EE_{P}[\ell((\x,\y);\param)].\label{eq:minimax}
  \end{equation}
The minimax objective captures many existing settings of interest, such as domain adaptation (where $\ppot$ is a small-number of target domains \cite{mansour2009domain}), uniform subgroup guarantees (where $\ppot$ are minority subgroups of the training distribution \cite{hashimoto2018repeated, duchi2018learning}), and distributionally robust optimization (where $\ppot$ is a divergence ball centered around the training distribution \cite{bental2013robust}).
We will focus on \emph{uniform subgroup guarantees} which define the set of potential test distributions as subpopulations with size at least $\atrue$,
\begin{multline}$$ \pjoint := \{\mathcal{Q}_0: \ptrain(\x,\y) = \atrue\mathcal{Q}_0(\x,\y)\\
  +(1-\atrue)\mathcal{Q}_1(\x,\y) \textrm{ for some }\mathcal{Q}_1\textrm{ with } \atrue > \adro\}$$.\label{eq:subpop}
\end{multline}
Prior work has shown that such shifts over groups can be used to capture a wide range of test time distributions  including label shifts \cite{hu2018does}, topics within a corpus \cite{oren2019drolm} , and demographic groups \cite{hashimoto2018repeated}. However, the set $\pjoint$ includes \emph{all} possible subpopulations which can be too pessimistic since this allows the conditional distribution $\ptest(\y\mid \x)$ to change arbitrarily and adversarially subject to the $\adro$-overlap constraint. For example, minimizing $\mc{R}(\param,\pjoint)$ with the zero-one loss results in a degenerate worst-case group that simply groups all the misclassified examples (up to a $\adro$ fraction) into an adversarial worst-case group \cite{hu2018does}. This drawback will lead us to consider restricted forms of the subpopulation guarantee in $\pjoint$.

A common approach for avoiding such degeneracy is to make a \emph{covariate shift} assumption \cite{shimodaira2000improving,quinonero2009dataset} which asserts that
\begin{equation}\label{eq:covshift}
  \ptrain(\y \mid \x) = \ptest(\y \mid \x).
  \end{equation}
  This resolves the earlier issues by restricting the subpopulation shift~\eqref{eq:subpop} to the covariate $x$. We will define this uncertainty set as $\pcov$ analogously to $\pjoint$.
    One particularly appealing property of covariate shift is that the Bayes-optimal classifier on $\ptrain$ will be Bayes-optimal on any $\ptest\in \pcov$, making it possible to simultaneously perform well on both the average \emph{and} worst case.  Unfortunately, this assumption is usually violated, as many distributional shifts involve unmeasured variables $\conf$ and even if $\y \mid \x, \conf$ remains fixed across train and test, the same may not hold for $\y \mid \x$.

    \subsection{Covariate shifts over unobserved variables}
    Recall our earlier example of a model trained on the general population to predict heart disease ($y$) from features such as obesity ($x$) and tested on recent smokers in a rehabilitation center. The covariate shift assumption does not hold, as the conditional distribution $\y\mid \x$ differs substantially from training to test. This example of \emph{omitted variable bias} arises whenever we fail to account for confounders, mediators, and effect modifiers \citep{angrist2009econometrics,van2015causal}.

    Using a general purpose uncertainty set such as $\pjoint$ to capture these types of shifts would also allow for nearly arbitrary shifts in the predictive distribution and would prevent us from making any predictions at all. However, the situation changes drastically if we observed whether individuals were smokers. If smoking is the only unmeasured variable which changes between train and test, $\y\mid \x,\conf$ remains fixed and this allows us to make predictions based only on correlations between $\y$ and $\x$ which remain reliable under distributional shifts on $\conf$. 

    Making this intuition precise, we will require that $c$ make the train and test distributions differ by a covariate shift in $(x,c)$, \begin{equation}\label{eq:covshift}
  \ptrain(\y \mid \x, \conf) = \ptest(\y \mid \x, \conf).
\end{equation}
This criterion (known as exogeneity in \citet{pearl2000causality}) defines our desired set $\conf$; however, this definition neither guarantees the existence of $c$ nor allows us to find a valid $\conf$ for a given generative mechanism. We will now show how to identify valid unmeasured variables $\conf$ under a given graphical model. 

\subsection{Conditions given a graphical model } \label{fig:condition}

Suppose that the features and labels $\x, \y$ are associated with a probabilistic graphical model that captures the generative process of $x$ and $y$. Now define a selector variable $z$ which determines whether a sample is included in the train or test distribution, with edges in the graph consistent with the covariate shifts (i.e. $\ptrain(\x,\conf,\y) = p(\x,\conf,\y,z=0)$ and $\ptest(\x,\conf,\y) = p(\x,\conf,\y,z=1)$).

We now state a necessary and sufficient condition for $\conf$ to fulfill~\eqref{eq:covshift}, which is that $\conf$ consists of all variables such that $y$ is d-separated from $z$ by $(x, \conf)$,
\begin{prop}
  A set of variables $\conf$ in a causal graph fulfills the exogeneity condition~\eqref{eq:covshift} whenever $z$ and $y$ are d-separated by $(x,c)$.
\end{prop}
This follows from the definitions of exogeneity and d-separation, which imply $p(\y\mid\x,\conf,z)=p(\y\mid\x,\conf)$. When the graph has a causal interpretation, and $z$ has no children\footnote{This common situation (which is often referred to as sample selection bias) occurs whenever the data already exists, and the training and test distributions are constructed by sampling and selecting examples from a population.} (Figure~\ref{fig:causal}), $z$ acts as a treatment indicator and $c$ is the set of confounders for the effect of $z$ on $y$ conditional $x$. This follows from the fact that d-separation and blocking backdoor paths are equivalent if $z$ has no children \cite{van2013confounder}.
\begin{figure}[ht] 
  \centering
\includegraphics[scale=0.38]{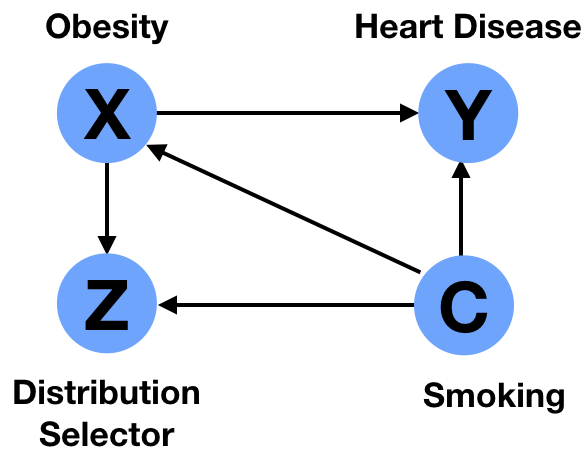}
\caption{Smoking ($c$) d-separates whether an example is in the test or train set ($z$) from the heart disease label ($y$) conditioned on obesity ($x$). In this example $z$ has no children since we select rather than generate examples, and assuming this is the true causal graph, $c$ is a confounder for the effect of $z$ on $y$.} 
  \label{fig:causal}
\end{figure}

\subsection{Sampling Unmeasured Variables} 
Our main challenge is that even if we know unmeasured variables $\conf$ exist, we cannot measure their value on our training data and use $\conf$ to constrain the test time conditionals $\ptest(\y\mid\x)$. However, if we could sample from the distribution $p(\conf \mid \x,\y)$, and combine this with $(\x,\y)$ samples in the training data, we can obtain samples from the full joint distribution of $(\conf,\x,\y)$. This distribution would, in turn, allow us to understand the set of potential $\ptest$ that can occur when we shift the marginal distribution of $(\x, \conf)$.

This key point will allow us to reduce the domain adaptation problem over $(\x,\y)$ to a covariate shift problem over features $(\x,\conf)$. Robustness under covariate shift is still challenging, but we can now apply existing techniques from the covariate shift literature, such as likelihood re-weighting \citep{shimodaira2000improving} or distributionally robust optimization (DRO) \citep{duchi2019distributionally}.

The main insight of this paper is that we can approximate this conditional distribution ($\conf\mid \x,\y$) through \emph{human annotation}: we ask human annotators for ``additional reasons'' why feature $\x$ would lead to  $\y$, and record the natural language explanations as approximate unmeasured variables $\ol{\conf}$. 
A key property of our human annotation procedure is that it is only used to augment \emph{training} data. We cannot sample from this same conditional distribution at \emph{test time}, since we do not have $\y$, and sampling from $\conf\mid \x$ does not provide any additional information beyond $\x$. Our proposed procedure therefore does not rely on any annotation of unmeasured variables $\conf$ at test time. We use shifts over the elicited $\ol{\conf}$ at training time to determine potential shifts in $(x,\ol{c})$, and learn a model that uses only observable features $\x$ that is robust to these shifts.

Conceptually, our approach has three parts: we first elicit $\ol{\conf}\mid \x,\y$ from human annotators over our training data. We then use $(\x,\ol{\conf},\y)$ to define the potential test-time shifts $\ptest(\x,\ol{\conf}) \in \mathcal{P}$. Finally, we learn a model $\param$ that predicts $\x\to\y$ such that $\ell(\x,\y;\param)$ is small over the entirety of $\mathcal{P}$.

\section{Estimation and Optimization}\label{method}
We now discuss the challenges of learning a robust model over unmeasured variables $\conf$. We first define our estimator in terms of the true unmeasured variable $\conf$ and later discuss the challenges associated with using elicited $\ol{\conf}$ in Section \ref{sec:approxconf}.

Given samples from $p(\conf \mid \x, \y)$, we can consider the covariate shift problem over $(\x,\conf)$. In principle, our proposal can utilize any covariate shift approach. However, to illustrate concrete performance improvements for our proposal, we will focus on the uniform subpopulation setting with distributionally robust optimization.

Adapting the earlier covariate subpopulation uncertainty set $\pcov$ to this case, we obtain uncertainty sets defined over $(\x,\conf)$,
 \begin{multline} 
 $$\pconf := \{\mathcal{Q}_0: \ptrain(\x,\conf) = \atrue\mathcal{Q}_0(\x,\conf)\\+(1-\atrue)\mathcal{Q}_1(\x,\conf)  \textrm{ for some }\mathcal{Q}_1\textrm{ with } \atrue > \adro\}$$.
 \end{multline}
 \break
 The resulting distributionally robust objective is now
 \begin{equation}
 \inf_{\param \in \param}\sup_{Q_0 \in \pconf}\EE_{\x,\conf \sim Q_0}[\EE[\ell(\param;(\x, \y)) | \x, \conf]]. 
 \label{eq:uvdro}
 \end{equation}
 We refer to the distributionally robust optimization problem over this uncertainty set ($\mathcal{R}(\param, \pconf)$) as distributionally robust optimization over shifts in unmeasured variables, or (UV-DRO). Although there exist many techniques for efficient distributionally robust optimization \cite{bental2013robust,namkoong2016stochastic,duchi2018learning}, the UV-DRO objective is challenging to estimate from finite samples, as the outer supremum depends on the \emph{conditional} risk $\EE[\ell(\param;(\x, \y)) | \x, \conf]$ rather than the loss $\ell(\param;(\x,\y))$.  

\paragraph{Finite Sample Estimation} Having defined the UV-DRO objective in terms of the population expectations, we now turn to the question of estimating this objective from finite samples. As we have mentioned earlier,
the empirical plug-in estimator fails to provide tight bounds for UV-DRO, and a Jensen's inequality argument shows that a naive plug-in UV-DRO estimator ignores the covariate shift structure, resulting in an estimator that is equivalent to the worst-case subpopulation objective over $\pjoint$.

One straightforward way to sidestep this challenge is to make smoothness assumptions on $\confcondrisk$. Let the $L_2$-normalized conditional risk $\frac{\hinge{\E[\ell(\theta;(\x, \y)) | \x, \conf]-\eta}}{\norm{\hinge{\E[\ell(\theta;(\x, \y)) | \x, \conf]-\eta}]}_2}$ be $L$-Lipschitz,\footnote{Incorporating smoothness acknowledges the fact that
  real world domain adaptation tasks are not arbitrary, and similar examples suffer similar conditional risk. This prior knowledge can help reduce the effective dimensionality of our inputs $\ol{c}$.} and $\mc{H}_L$ be the set of $L$-Lipschitz positive functions with $L_2$ norm less than one. Standard variational arguments for distributionally robust optimization in \citet{duchi2019distributionally} give the following variational upper bound:
\begin{align*}
 \small &R(\theta,\pconf)=\inf_{\eta}\frac{1}{\adro}\E[\hinge{\E[\ell(\theta;(\x, \y)) | \x, \conf]-\eta}]+\eta\\
  &\quad\leq \inf_{\eta} \sup_{h\in\mc{H}_L} \frac{1}{\adro}\E[h(\x,\conf)\left(\E[\ell(\theta;(\x, \y)) | \x, \conf]-\eta\right)]+\eta\\
  &\quad=: R_L(\theta)
\end{align*}
where the expectations are taken with respect to $\ptrain$.
The first step follows from standard convex duality for distributional robustness, while the second follows from the variational form of the $L_2$ norm. The dual form of $R_L$ has
a simple empirical plug-in estimator,

  \resizebox{0.47\textwidth}{!}{
    \begin{minipage}{0.57\textwidth}
\begin{equation}\label{eq:finiteobj}
    \begin{split}
      & \inf_{B,\eta\ge 0} \frac{1}{\adro}\bigg( \frac{1}{n} \sum_{i=1}^n
      \hingeBig{ \loss(\param; (\x_i,\y_i))
        -\sum_{j=1}^n (B_{ij} - B_{ji}) -\eta}^2
      \bigg)^{1/2} \\
      & \qquad  + \frac{L}{ n}
      \sum_{i,j = 1}^n \left(\norm{\x_i - \x_j}+\norm{\conf_i-\conf_j}\right) B_{ij} + \eta
      .
    \end{split}
  \end{equation}

\end{minipage}
}
For detailed derivation, see the appendix. This is a special case of the family of $L_p$ norm variational DRO estimators proposed and studied by \citet{duchi2019distributionally} and is known to converge to its population counterpart at rate $O(n^{-1/d})$. 

This estimator intuitively captures both the smoothness assumption and worst-case structure of our objective. The dual variable $\eta$ serves as a cutoff: all losses $\ell$ below $\eta$ within the sum are set to zero, forcing the model to focus on the worst losses incurred by the model. The dual variable $B$ is a transport matrix, where the entry $B_{ij}$ transports loss from example $i$ to $j$ in exchange for a cost of $L(\norm{\x_i-\x_j}+\norm{\conf_i-\conf_j})$. This smoothing ensures that the model focuses its attention on neighborhoods of the input $(\x,\conf)$ that systematically have high losses.

\section{Approximation with Crowdsourcing}\label{crowdsource}
\paragraph*{Effect of Approximating Unmeasured Variables}
\label{sec:approxconf}
Minimizing the UV-DRO objective~\eqref{eq:finiteobj} with $c$ provides a model which is robust to test time shifts that potentially change $y\mid x$. Unfortunately, we do not have access to the \emph{unmeasured} variables $c$, and instead only observe noisy and approximate samples $\ol{c}\mid x,y$ (e.g. natural language explanation from human crowdworkers).

We now characterize the conditions under which a model estimated using approximate unmeasured variables ($\ol{R}_L$) performs well on the true risk ($R_L$) under the unmeasured variable $c$. A major challenge in comparing approximate and true unmeasured variables is that $\ol{c}\in\olmc{C}$ and $c\in\mc{C}$ are unlikely to even exist in the same metric space. 

We overcome this difficulty by characterizing the risk in terms of an alignment. If there exists smooth functions $f$ and $g$ which align the space of approximate unmeasured variables $\olmc{C}$ with the space of true unmeasured variables $\mc{C}$, then the optimal model under the approximate $\ol{c}$ performs well on the true risk $R_L$.

\begin{prop}\label{prop:align}
  Let $f: \mc{C} \to \olmc{C}$ and $g: \olmc{C} \to \mc{C}$ be any $K_f$ and $K_g$ Lipschitz-continuous functions. For positive losses bounded above by $M$, the minimizer for the approximate risk ($\ol{R}_L$) given by 
  $\ol{\theta}^* := \arg\min_\theta \ol{R}_{L}(\theta)$ 
  fulfills
  \begin{align*}
    &R_L(\ol{\theta}^*) - \inf_\theta R_{L}(\theta) \\
    &\leq \inf_\theta R_{L}(\theta)\left(K_f K_g-1\right)+\frac{LM}{\alpha}(A_fK_g + A_g)
  \end{align*}
  where
  \begin{align*}
    \small
    A_f = \EE W_1(\ol{c}|xy, f(c)|xy) ~~\text{and}~~
    A_g = \EE W_1(c|xy, g(\ol{c})|xy)
  \end{align*}
  and $W_1(\ol{c},f(c))$ is the Wasserstein distance between the distribution of $\ol{c}$ and the pushforward measure of $c$ under $f$. 
\end{prop}
See the appendix for proofs and additional bounds.

The $K_fK_g$ distortion term captures the fact that a Lipschitz continuity assumption under $\ol{c}$ differs from one under $c$. The additive terms $A_f,A_g$ captures the distributional differences between $c|xy$ and $\ol{c}|xy$. Note that if $\ol{c}$ and $c$ share the same metric space, the relative error term $(K_fK_g-1)$ is zero, and the model approximation quality depends on the average Wasserstein distance between $c\mid xy$ and $\ol{c}\mid xy$. 

\paragraph{Crowdsourcing for Elicitation}
To better understand the approximation bound, consider the example of a digit recognition task confounded by rotation, where we are asked to classify images ($x$) of digits which have undergone an unobserved rotation ($c$). The unmeasured variable is a real-valued angle, but $\ol{c}$ is a natural language annotation with the metric defined by vector embeddings of sentences (Figure \ref{fig:confcap}).

In this example, the distortion term $K_fK_g$ captures whether the distance between natural language description of two images whose rotations differ by $d$ degrees is close to $d$. The Wasserstein term captures the fact that natural language descriptions can be noisy, and we can sometimes get annotations that do not correspond to any $c$.

We attempt to mitigate the effect of two terms through the crowdsourcing design:
\begin{enumerate}
  \itemsep=0pt
  \item Each user annotated many examples to reduce phrasing variation (a person using ``turn" instead of ``rotate" will likely only refer to rotations as ``turn'', reducing annotator variation).
	\item We selected a vector sentence representation (Sent2Vec) whose distances have been shown to correlate with semantic similarity.
	\item We collected and averaged multiple annotations per training example to reduce crowdsourcing noise.

\end{enumerate}
This data collection procedure allows us to capture the well-studied ability for humans to identify unobserved causes \cite{schulz2008unobserved, saxe2007causal} and allows us to obtain more robust models under several types of bias from unmeasured variables.

\section{Experimental Results} \label{results}
We now demonstrate that distributional robustness over unmeasured variables (UV-DRO) results in more robust models that rely less upon spurious correlations. Across all experiments, we show that UV-DRO achieves more robust models than baselines as well as other DRO objectives, including that of \citet{duchi2019distributionally} (``Covariate Shift DRO") and \citet{hashimoto2018repeated} (``Baseline DRO").
\paragraph*{Experimental Procedures.} Both the linear regression (Section~\ref{sec:simdata}) and logistic regression models (Sections~\ref{mnist} and \ref{sec:stopfrisk}) were optimized using batch gradient descent with AdaGrad. We tuned hyperparameters such as the learning rate, regularization, and DRO parameters using a held-out validation set, which we describe in the appendix.

Human annotations were performed by crowdworkers on Amazon Mechanical Turk, and the specific prompts for each task are included in the relevant sections. In both tasks, we define the distance between two annotations $c$ by embedding each sentence into vector space with the FastText Sent2Vec library, and measuring the average cosine distance between the two vectors across two replicate annotations.

\subsection{Simulated Medical Diagnosis Task}
\label{sec:simdata}
We begin with a simple simulated medical diagnosis dataset. One source of bias in medical datasets is that patients can sometimes lie about symptoms to their doctors. It has been well-studied that adolescents have a substantially higher chance of lying to physicians about sexual activity \cite{zhao2016teen}, which complicates medical diagnosis and the ability to prescribe teratogenic drugs such as the acne drug Accutane \citep{honein2001accutane}. In this example, age is an effect modifier which for simplicity we assume is the only unmeasured variable \citep{van2012effect}.

We consider a simplified scenario of using the patient's self-reported pregnancy symptoms $\x_1$ and clinical measurements $\x_2$ to predict pregnancy $\y$ via a least-squares linear regression model ($\y=\beta^\top\x+b$).
We demonstrate that a model trained with empirical risk minimization (ERM) learns the unreliable correlation between self-reporting and pregnancy, while UV-DRO using an imputed age learns to use noisier but more reliable clinical measurements. 

We will define the data generating distribution for our observations as $\x_1 = \conf\y$ where $\conf \sim 1-2~\text{Bernoulli}(q)$ is the patient's truthfulness, and the clinical measurements follow $\x_2 = \y+\epsilon$ where $\epsilon \sim N(0,4)$ is a measurement noise term. We evaluate the models on a series of training distributions with a mix of adults and adolescents where the probability of lying ranges over  $q_{train} = \{.05, .1, .2, .3, .4, .5, .6, .7, .8\}$. At test time, our model is applied to adolescents who lie with probability $q_{test}=0.8$. These datasets fulfill the subpopulation condition (i.e. $\pconf$) where the train-test overlap varies over $\atrue = \{.0625, .125, .250, .375, .5, .625, .750, .875, 1.0\}$.

From our data generating distribution, we can see that $\conf$ is an unmeasured variable which affects the correlation between $\y$ and $\x_1$. During training time, where the patient set largely consists of adults that are less likely to lie, $\conf$ is often 1, and a model optimized on this data will predict $\y$ using primarily $\x_1$, as shown by the low relative weight of $\x_2$ in Figure \ref{toyexample_weights}. However, at test time when there are many adolescent patients who have high likelihood of lying, the correlation between $\x_1$ and $\y$ is reversed, making this model perform poorly on the test set  (Figure \ref{toyexample_loss}).

On the other hand, if we apply UV-DRO with $\conf$ sampled according to the true conditional distribution $\conf \mid \x_1,\x_2,\y$, then our loss will account for the fact that the correlation between $\x_1$ and $\y$ may flip at test time, and our learned model uses $\x_2$  (Figure \ref{toyexample_weights})---which reliably measures $\y$. This results in substantial gains in test performance that are stable across a wide range of $\atrue$s (Figure \ref{toyexample_loss}). Finally, we observe that conditioning on $y$ is critical when generating $c$, and using $\conf \mid \x$ instead, as would be the case if we sampled $c$ at test time, fails to improve robustness (Figure \ref{toyexample_loss}).

\begin{figure}[h] 
  \centering
\includegraphics[scale=0.1]{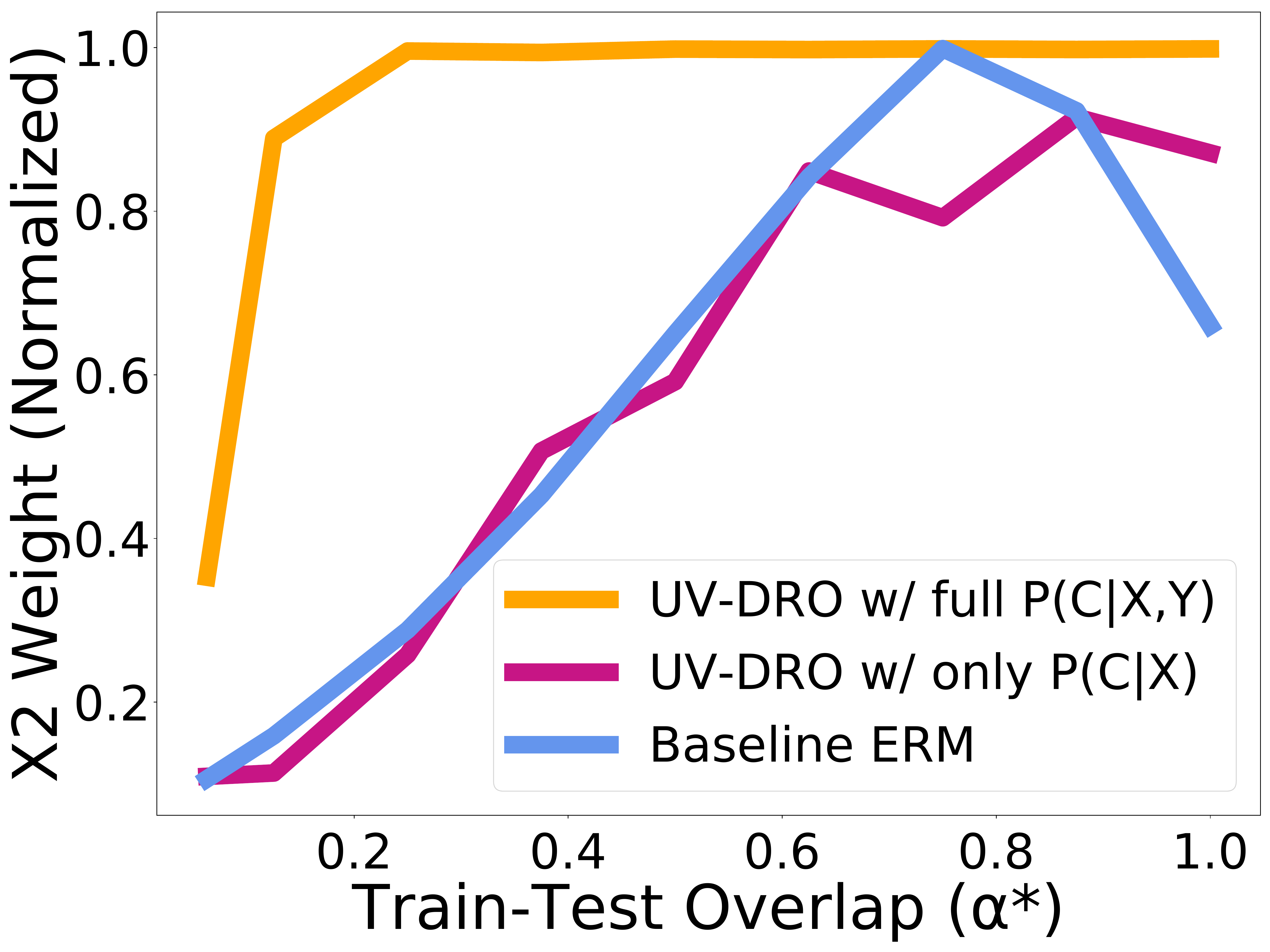}
\caption{UV-DRO consistently places higher relative weight on more reliable feature $\x_2$ over $\x_1$,  unlike both ERM and a baseline UV-DRO with $\conf$ drawn without access to label $\y$.}
\label{toyexample_weights}
\vspace{-5pt}
\end{figure}

\begin{figure}[h] 
  \centering
\includegraphics[scale=0.1]{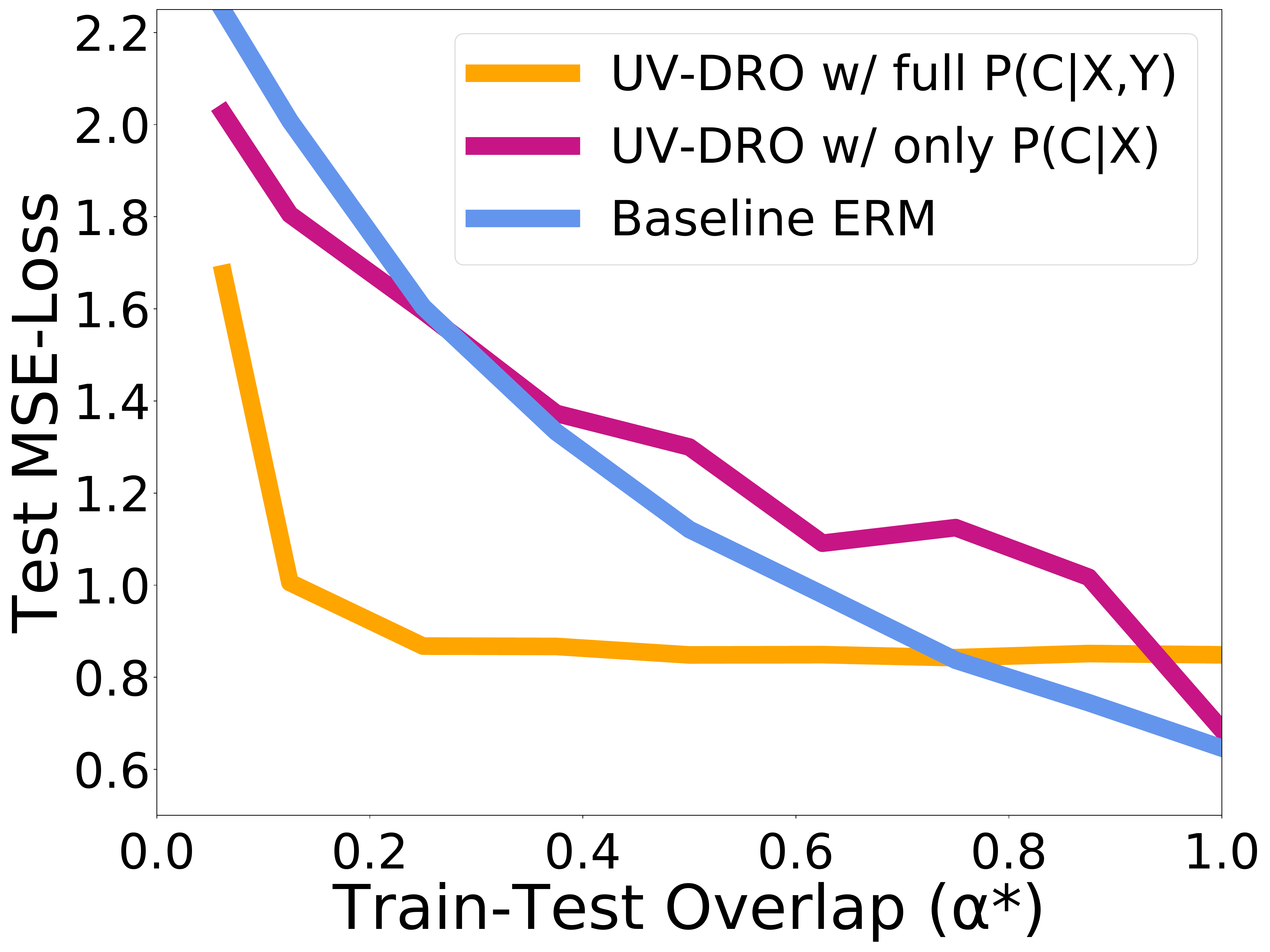}
\caption{UV-DRO achieves lower loss than both ERM and a baseline UV-DRO with $\conf$ drawn without access to label $\y$.}
\label{toyexample_loss}
\vspace{-5pt}
\end{figure}

\subsection{Digit Classification Under Transformations}
\label{mnist}

We evaluate the efficacy of UV-DRO on synthetic domain shifts on the MNIST digit classification task. Specifically, we apply random rotations or occlusions to the images and treat the identity of these transformations as an unmeasured variable. We show that if the distribution of such unmeasured variables shifts from training to test sets, classification accuracy for a simple logistic regression model degrades rapidly for both ERM and existing DRO approaches, but this performance loss is mitigated when using UV-DRO.

Examples of these image transformations, as well as crowdsourced annotations, are shown in Figure \ref{fig:confcap}. We can see that digits such as rotated 6s and 9s can become difficult or impossible to distinguish without knowledge of the rotation angle, and our logistic regression model's performance rapidly degrades as the fraction of rotated digits changes between train and test sets.
\begin{figure}[h]
  \centering
  \includegraphics[scale=0.25]{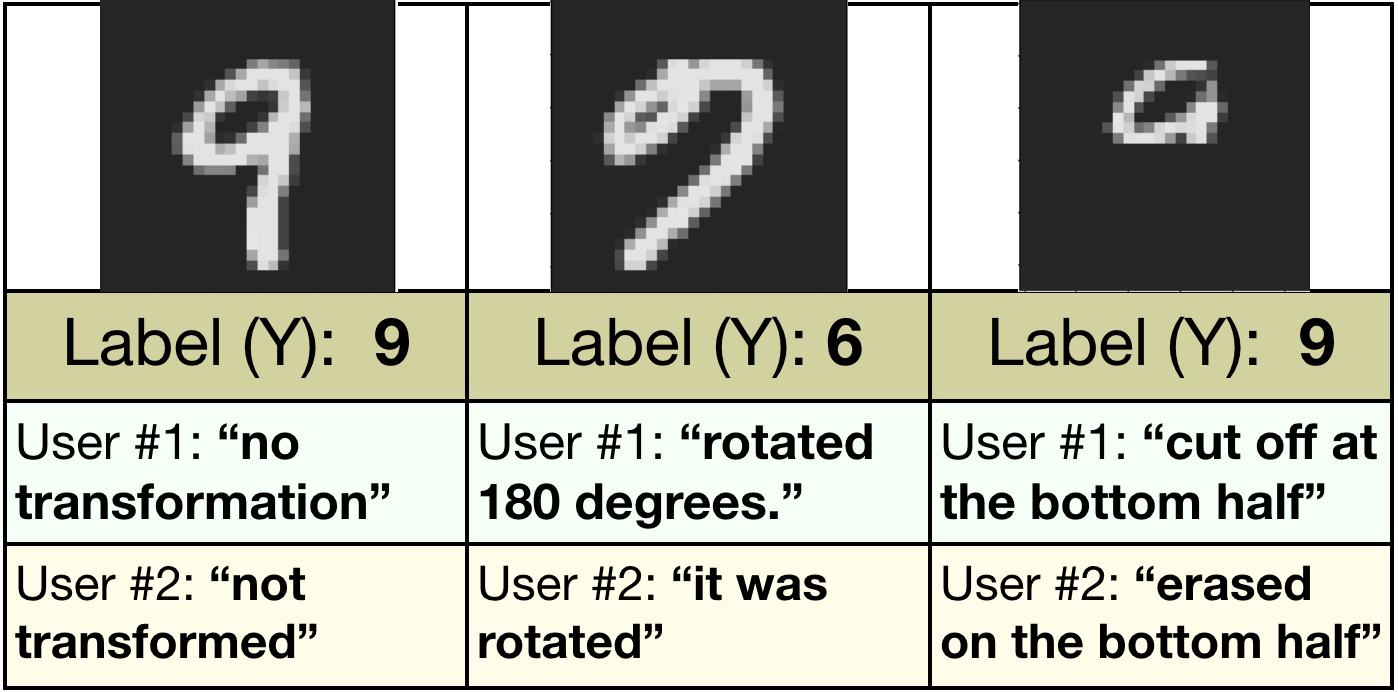}
  \caption{Examples from our training dataset as well as the user-provided annotations from our crowdsourcing task. Human crowdworkers are able to directly recover the unmeasured variables, leading to nearly oracle level performance for UV-DRO.}\label{fig:confcap}
  \vspace{-10pt}
  \end{figure}

\begin{figure*}[t]
  \centering
  \begin{subfigure}[t]{0.3\textwidth}
  \includegraphics[scale=0.1]{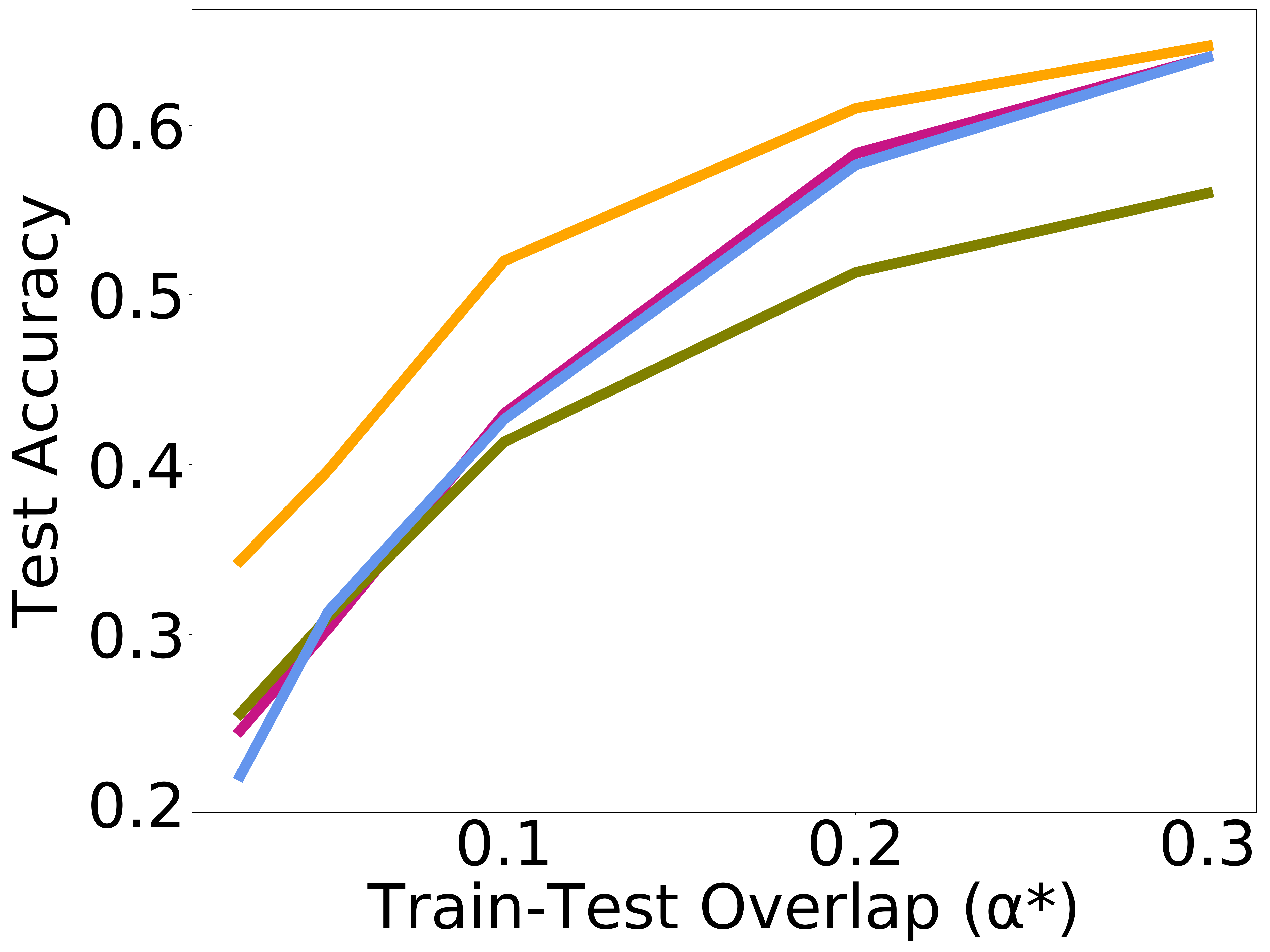}
  \caption{Accuracy vs. minority group size}\label{fig:mnistacc}
  \end{subfigure}
~
  \begin{subfigure}[t]{0.3\textwidth}
  \includegraphics[scale=0.1]{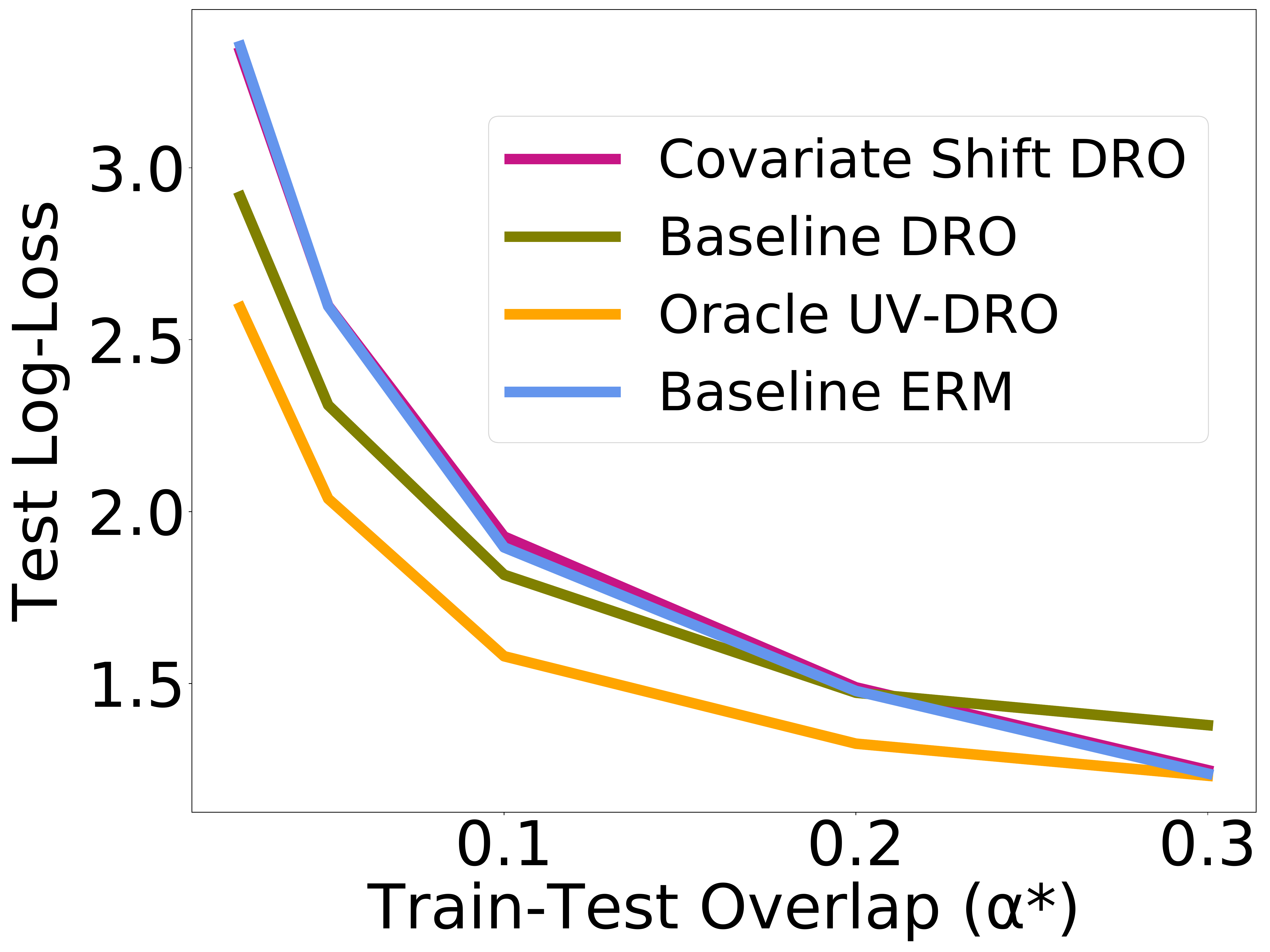}
  \caption{Log-Loss vs. minority group size}\label{fig:mnistloss}
\end{subfigure}
~
\begin{subfigure}[t]{0.3\textwidth}
  \includegraphics[scale=0.04]{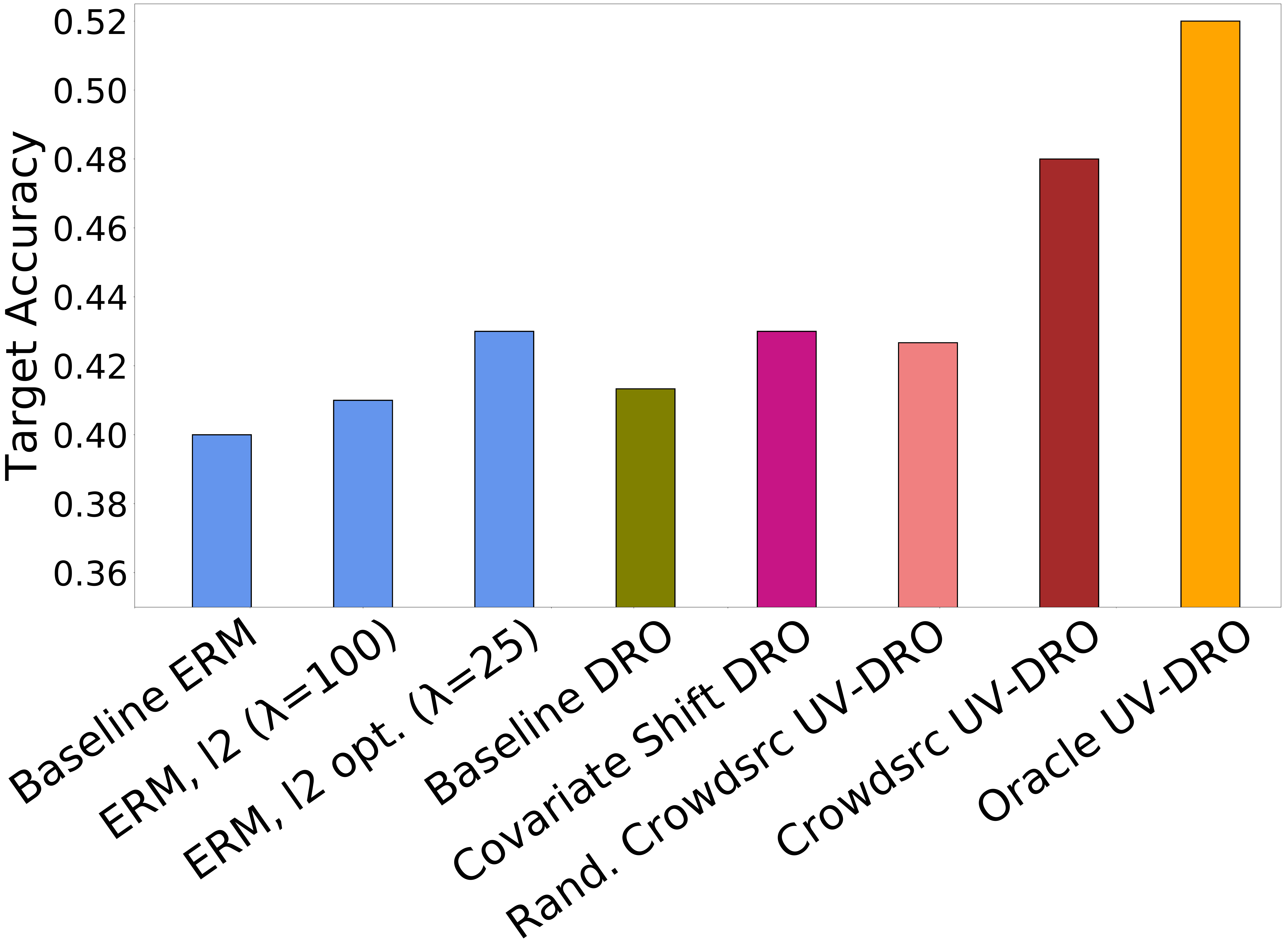}
  \caption{Accuracy with crowdsourced data}
  \label{fig:mnist_turk_acc}
  \end{subfigure}
  \caption{For the MNIST Digit Classification task, our UV-DRO approach using oracle unmeasured variable values results in substantial improvements in both accuracy and loss under large train test shifts ($\atrue$). Using crowdsourced annotations with our UV-DRO approach successfully provides an accuracy gain ($48\%$ UV-DRO vs. $43\%$ ERM) over the Baseline ERM models, as well as other DRO approaches.}
  \label{fig:mnistoracle}
  \end{figure*}
  
\begin{figure}[h!]
  \centering
  \includegraphics[scale=0.15]{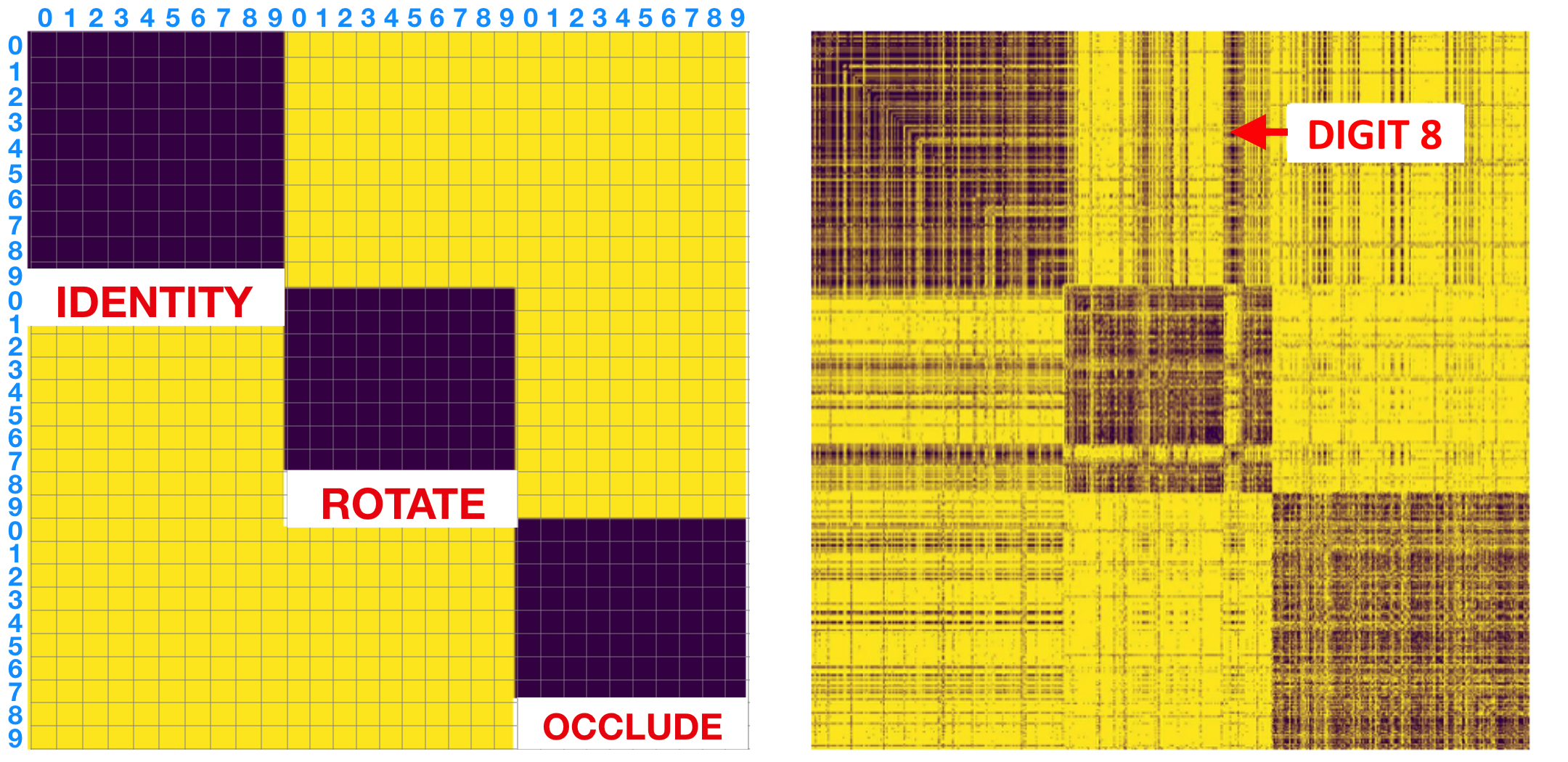}
  \caption{Cosine-distances between natural language annotations (right) strongly correlates with ground truth (left). The distance matrix is ordered by the transformation type and then digit. Note the  confusion between ``rotation" and ``identity" transform for digit 8, which is invariant to 180$^{\circ}$ rotations.}
  \label{fig:mnist_heatmap}
  \vspace{-12pt}
  \end{figure}
\paragraph{Estimation with an oracle}
In our first experiment, we consider the oracle setting, where the unmeasured variable $\conf$ is a rotation angle and we obtain samples from the joint distribution $(\x,\y,\conf)$ to use with UV-DRO for predicting $y$ given $x$. The training distribution consists of images of which a proportion $\atrue=\{0.05, 0.1, 0.2, 0.4, 0.6\}$ are transformed, while the test set consists only of images that are rotated by 180 degrees.

Unsurprisingly, we find low performance for baselines which do not use $\conf$ at small values of $\atrue$  (Figure \ref{fig:mnistacc}). This includes ERM (with and without $L_2$ regularization), DRO over $\pjoint$ (baseline DRO), and DRO over just the features $\pcov$ (covariate shift DRO). UV-DRO substantially improves performance over these baselines, with a $15\%$ \emph{absolute} accuracy gain for the most extreme shift of $\atrue=0.05$, as well as substantial accuracy gains that persist until $\atrue=0.6$ (Figure \ref{fig:mnistacc}). The same trend holds for the log-losses incurred by each model (Figure~\ref{fig:mnistloss}).

\paragraph{Estimation with crowdsourced unmeasured variables}
We next demonstrate that UV-DRO performs well even with crowdsourcing the unmeasured variables, and show this performance approaches that of the  oracle. Specifically, we consider a training distribution where images are either rotated with probability $0.1$, occluded with probability $0.1$, or not manipulated. The test distribution is the same as before, with all images rotated.

  In order to obtain samples from $p(\conf|\x,\y)$, we performed an Amazon Mechanical Turk task where crowdworkers were shown $(\x,\y)$ pairs of images ($\x$) and their label ($\y$) from our training dataset of 4000 images. Each user was prompted to answer a free-text question, ``What transformation do you think happened to the image?". Importantly, we \textit{did not} inform the users what types of unmeasured variables are present or possible in the dataset. We processed these natural language descriptions into a distance metric suitable for UV-DRO using our embedding procedure described earlier. We find that the resulting distances closely match the true unmeasured variable structure (Figure~\ref{fig:mnist_heatmap}).

  Training the UV-DRO model on this dataset, Figure ~\ref{fig:mnist_turk_acc} shows substantial accuracy improvement from crowdsourced unmeasured variables ($48\%$) compared to the ERM ($43\%$) and two DRO (41--43\%) baselines. Surprisingly, we also observe that crowdsourcing unmeasured variables results in only a $4\%$ accuracy drop relative to using oracle $\conf$'s ($52\%$), showing that we substantially close the gap to the optimal robust model. Finally, a randomly shuffled permutation of the annotation distances causes a significant drop in accuracy ($42\%$), showing that it is the crowdsourced information---rather than loss or hyperparameter changes---that results in performance gains. Further analysis on the effect of crowdsourcing quality is in the appendix.  

\subsection{Analyzing Policing Under Location Shifts}
\label{sec:stopfrisk}

\begin{figure*}[t]
  \begin{subfigure}{0.3\textwidth}
  \includegraphics[scale=0.09]{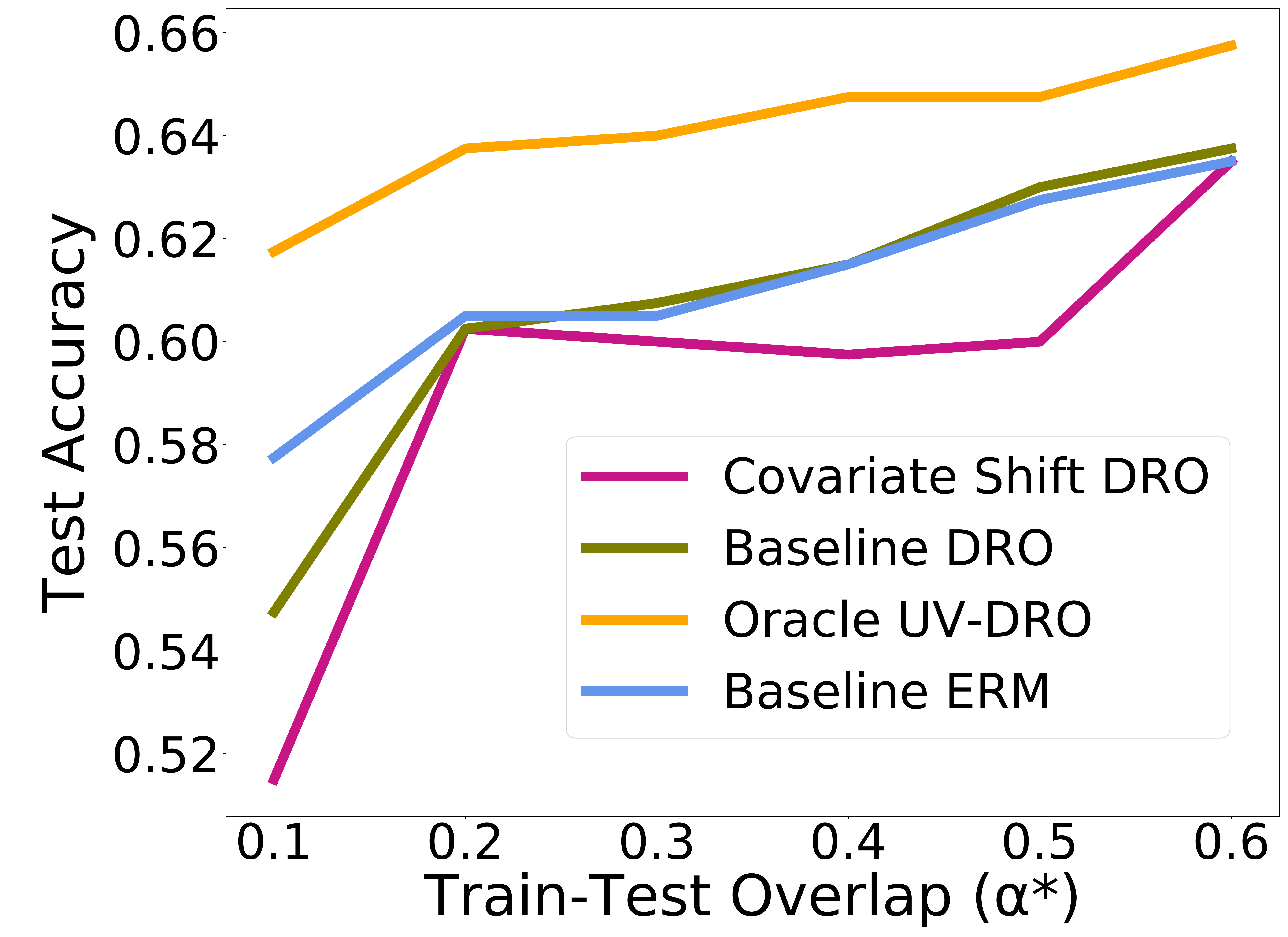}
  \caption{Accuracy vs. minority group size}
\end{subfigure}
  \begin{subfigure}{0.3\textwidth}
  \includegraphics[scale=0.09]{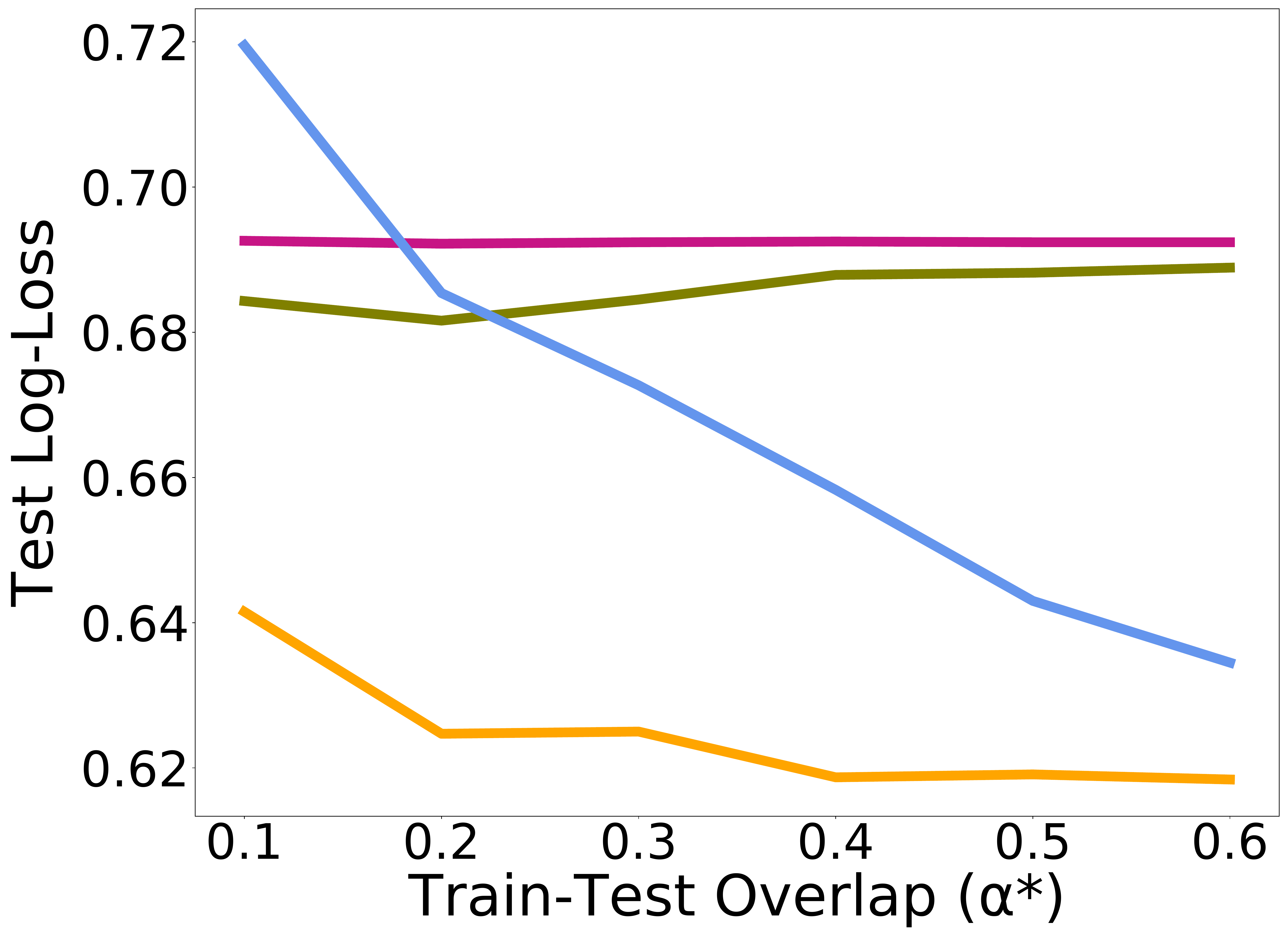}
  \caption{Log-Loss vs. minority group size}
\end{subfigure}  
  \begin{subfigure}{0.3\textwidth}
  \hspace{1em}
  \includegraphics[scale=0.045]{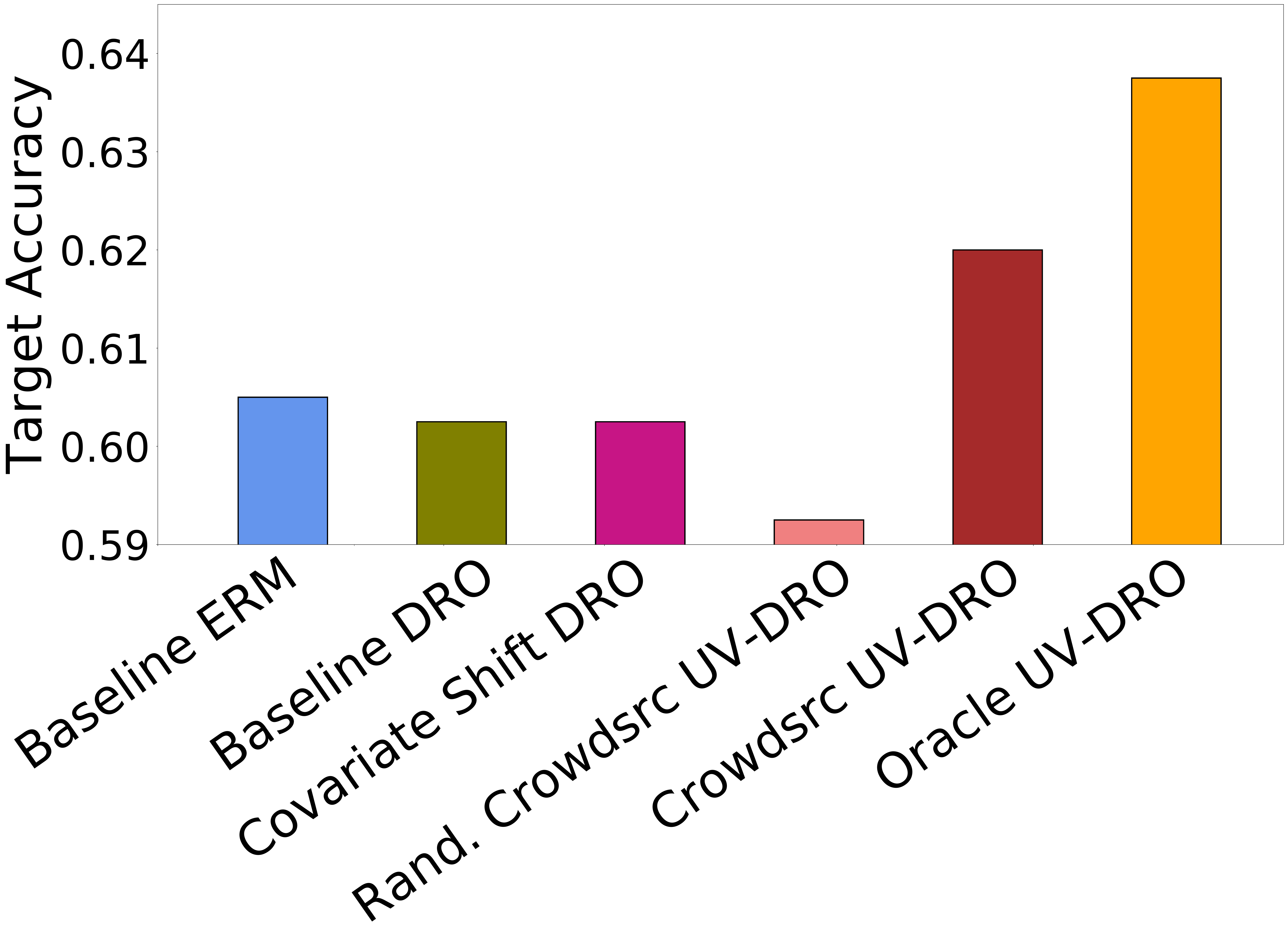}
  \caption{Accuracy on crowdsourced data }
  \end{subfigure}
  \caption{For the police stop analysis, our UV-DRO approach with oracle location variables provides consistent improvements in accuracy and loss. Using crowdsourced annotations with UV-DRO also improves accuracy ($62\%$ UV-DRO vs. $60.5\%$ ERM) over the Baseline ERM model and other DRO approaches. Only the unregularized ERM model is shown, as we found $\lambda=0$ to be optimal.}
  \vspace{-8pt}
  \label{fig:sf_oracle_acc}
\end{figure*}

Having demonstrated gains on the semi-synthetic MNIST task, we evaluate UV-DRO on a more complex real-world distribution shift.
Stop-and-frisk is a controversial program of temporarily stopping, questioning, or searching civilians by the police, and has been well-studied for amplifying racial biases. We consider the task of trying to detect false positives---or police stops that do not result in arrests---by training classifiers on data from police stops spanning 2003-2014 in New York City \citep{nyclu2019data}.

\begin{figure}[h]
  \centering
\includegraphics[scale=0.21]{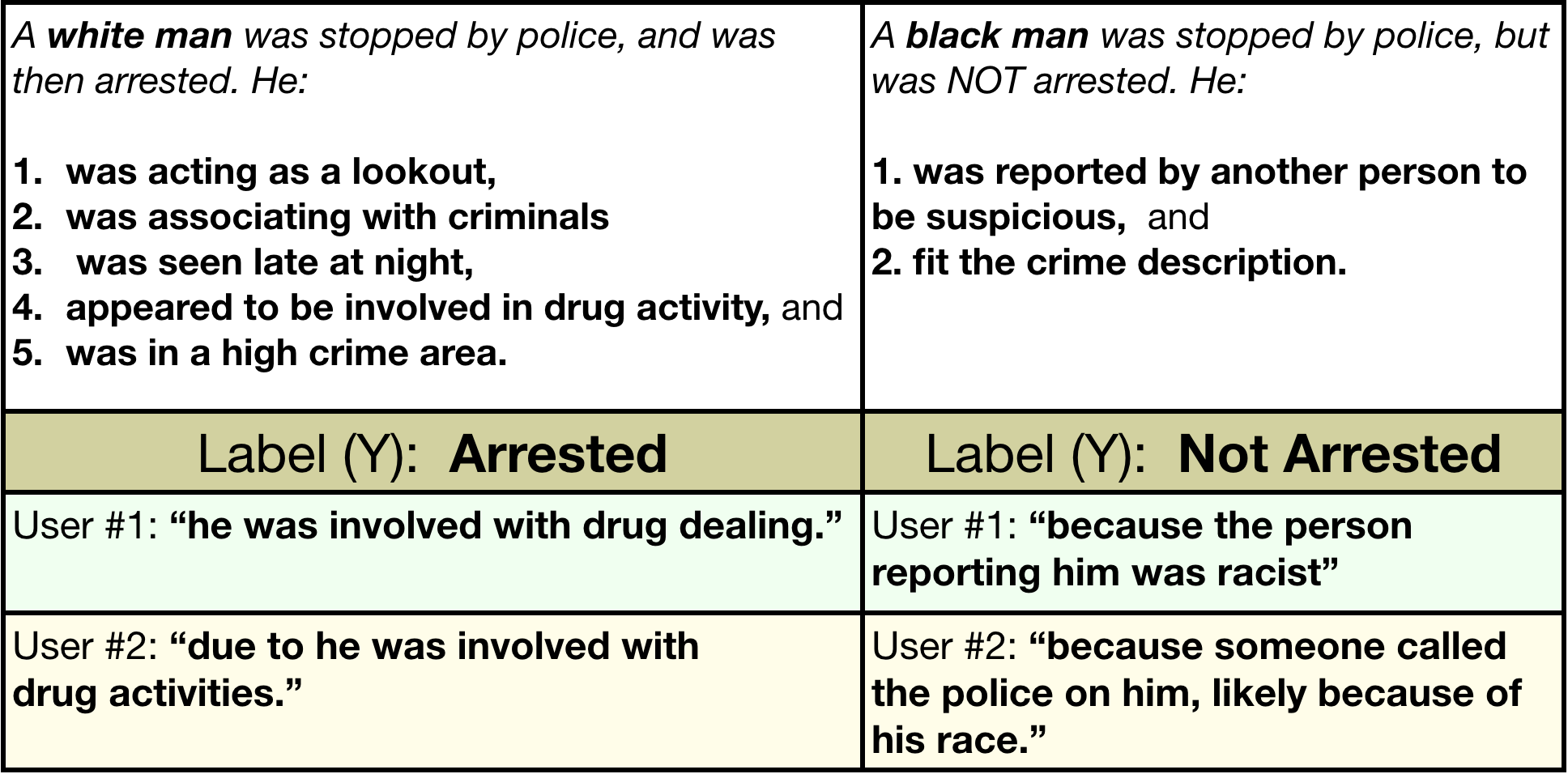}
\caption{Examples from the stop-and-frisk task identify perceived reliable features (left) as well as factors such as racism (right).}\label{fig:examplesf}
\vspace{-12pt}
  \end{figure}
For our observed features, we consider a set of 27 possible observations reported by police officers as reasons for stops, such as ``furtive movements'' and ``outline of weapon''. Examples from this dataset, as well as crowdsourced annotations, can be seen in Figure \ref{fig:examplesf}. Previous work on this dataset has shown that racial minorities are stopped more frequently for less serious observations \citep{goel2016stopandfrisk, gelman2007police}.

Our unmeasured variable in this setting is the location of the police stop.  We consider training data where the majority of police stops are from Manhattan, and measure the model performance on stops in Brooklyn. This is a natural form of domain shift, where our goal is for a model to make predictions that are reliable regardless of location.

\paragraph{Estimation with an oracle}
We first consider UV-DRO with an oracle, where the distribution of unmeasured variables matches the ground truth. The training distributions consists of police stops which occurred in Manhattan, except for a proportion $\atrue=\{0.1, 0.2, 0.3, 0.4, 0.5, 0.6\}$ which occurred in Brooklyn. The test distribution consists solely of police stops in Brooklyn.

Similar to our previous experiments, we see that this task is very challenging under large train-test shifts, with most models performing near 50--60\% accuracy. For Logistic Regression, using UV-DRO with oracle provides substantial gains on both accuracy and loss over a wide range of $\atrue$ (Figure ~\ref{fig:sf_oracle_acc}). For accuracy, UV-DRO provides just under a $4\%$ accuracy gain across most values of $\atrue$, while DRO baselines surprisingly do \emph{worse} than naively using ERM. 

\paragraph{Estimation with crowdsourced unmeasured variables}
Unlike the previous digit classification task, police stops and observations
can be affected by an incredibly large set of confounders. While we hypothesize that it is unlikely for crowdworkers to
actually predict location as an unmeasured confounder, Proposition~\ref{prop:align} suggests that this is not necessary. We show that crowdworkers capture a variety of social and demographic factors, which are sufficiently indicative location to provide robustness gains.  

For this task, we fix the proportion of training examples from Brooklyn at $\atrue=0.2$. We present Amazon Mechanical Turk users $(\x,\y)$ pairs of police stop descriptions---including the full feature set of race, gender, and officer observations ($\x$)---and the label of whether the individual was arrested or not ($\y$). Each user is then asked, ``What factors do you think led to the individual being stopped and [arrested/not arrested]?", for which they provide free-text responses. We use the same procedure as the MNIST experiment to process these responses into a distance matrix. 

Figure ~\ref{fig:examplesf} includes example annotations we elicited from users. Many free-text annotations showed strong ability to recover social factors (i.e. racism), as well as filter features to identify the relevant factors for an arrest decision. 

 Training our UV-DRO model, we find that crowdsourced UV-DRO gives a 1.5\% accuracy gain, which is once again nearly half of the gap between the ERM baseline and oracle DRO (Figure ~\ref{fig:sf_oracle_acc}). Similar to our previous experiment, we find that existing DRO baselines, as well as shuffling our crowdsourcing data, result in no gains over ERM.

Finally, to understand how the crowdsourced annotations capture unmeasured variables, we trained a logistic regression model to predict the police stop location from (1) only observed features (61.3\% test accuracy), (2) observed features and annotation unigrams (64.8\%), and (3) observed features and arrest labels (65.9\%). These results confirm that the annotations indeed help provide more information about the location than the observed features. Further exploratory analysis on predicting the location from annotation unigrams results in assigned weights that are consistent in showing that race, police judgement, and individual circumstances were more predictive of Brooklyn while unigrams associated with violent crime were more predictive of Manhattan (See appendix). This suggests that UV-DRO can not only be used to improve model robustness, but also has the potential to improve interpretability by highlighting unmeasured variables that may result in spurious correlations.

\section{Related Works and Discussion} \label{discussion}

Although our work draws on ideas from domain adaptation \citep{mansour2009dams, mansour2009domain, bendavid2006analysis} causal invariance \citep{peters2016causal, meinshausen2015maximin, rothenhausler2018anchor}, and robust optimization \citep{bental2013robust, duchi2018learning, bertsimas2018data, lam2015quantifying}, few prior works seek to elicit information on the possible shifts in unmeasured variables. For example, \citet{heinze2017conditional} improve robustness under unobserved style shifts in images by relying on multiple images which vary only by style. Similarly, \citet{landeiro2016confounder} develop a back-door adjustment which controls for confounding in text classification tasks when the confounder is known. Recent work by \citet{kaushik2019difference} use crowdsourcing  to revise document text given a specific counterfactual label. This approach is complementary to our work. \citet{kaushik2019difference} seek to expand a dataset by collecting additional examples with flipped labels (resulting in better models over observed variables), while our work augments existing data as a way to capture potential distribution shifts over unmeasured variables.

The crowdsourcing aspect of our work builds on existing work on eliciting human commonsense understanding and counterfactual reasoning. \citet{roemmele2011copa} showed that humans achieve high performance on commonsense causal reasoning tasks, while \citet{sap2019atomic} has used crowdsourcing to build an ``if-then" commonsense reasoning dataset. These works support our results which show crowdsourcing can successfully capture how humans reason about unmeasured variables. 

 \paragraph*{Discussion}
 We have demonstrated that domain adaptation problems with unmeasured variables can be recast as covariate shift problems once we obtain samples from $\conf \mid \x,\y$ at train time. Notably, rather than expanding the training dataset, our work accounts for variables that may be inaccessible in an already existing dataset. Our UV-DRO approach and experiments show that crowdsourcing can be an effective way of eliciting these unmeasured variables, and we often obtain results close to an oracle model which uses the true $\conf$ distribution. This work is the first step towards explicitly incorporating human knowledge of potential unmeasured variables via natural language annotations, and opens the possibility of methods that make use of counterfactual explanations from domain experts to learn reliable models in high-stakes situations.  

 \paragraph*{Reproducibility}
 We provide all source code, data, and experiments as part of a worksheet on the CodaLab platform: \url{https://bit.ly/uvdro-codalab}.
 
 \paragraph*{Acknowledgements}
 We thank all anonymous reviewers for their valuable feedback. Toyota Research Institute (``TRI")  provided funds to assist the authors with their research but this article solely reflects the opinions and conclusions of its authors and not TRI or any other Toyota entity. MS was additionally supported by the NSF Graduate Research Fellowship Program under Grant No. DGE-1656518.  

\bibliography{all.bib}
\bibliographystyle{icml2020}

\appendix
\onecolumn
\counterwithin{figure}{section}
\counterwithin{table}{section}

\label{section:proof-conf-empirical-dual}

\newcommand{\R}{\mathbb{R}}
\newcommand{\lossitem}{\ell(\theta;(x_i,y_i))}
\newcommand{\subjectto}{s.t.}
\newcommand{\maximize}{\max}
\providecommand{\tr}{\mathop{\rm tr}}
\newcommand{\onevec}{\mathbf{1}}
\newcommand{\confthresh}{\varepsilon}
\newcommand{\opt}{^*}

\break
\break
\break
\break
\break
\break
\section{Derivation of the empirical dual estimator}
The arguments given here are a simplification of the class of duality arguments from \citet{duchi2019distributionally}.
Recall that the inner maximization $\sup_{h\in \mathcal{H}_L} \EE[h(x,c) (\EE[\ell(\theta;(x,y))|x,c]-\eta)]$ admits a plug-in estimator which can be written as a linear objective with Lipschitz smoothness and $L_2$ norm constraints,
\begin{align}
  \label{eq:primal-again}
  & \maximize_{h \in \R^n}
    \frac{1}{n} \sum_{i=1}^n h_i (\lossitem - \eta) \\
  & \subjectto ~~h_i \ge 0
    ~~\mbox{for all}~~i \in [n], 
    ~~\frac{1}{n} \sum_{i=1}^n h_i^2 \le 1, \nonumber \\
  & \hspace{45pt}
    ~~h_i - h_j
    \le L (\norm{x_i - x_j}+\norm{c_i-c_j})~\mbox{for all}~~i,j \in [n].
    \nonumber
\end{align}    

Now taking the dual with $\gamma \in \R^n_+$, $\lambda \ge 0$, and $B \in \R^{n \times n}_+$, the associated Lagrangian is
\begin{align*}
  \mc{L}(h, \gamma, \lambda, B)
  & \defeq
    \frac{1}{n} \sum_{i=1}^n h_i (\lossitem-\eta)
    + \frac{1}{n} \gamma^\top h
    + \frac{\lambda}{2} \left( 1- \frac{1}{n} \sum_{i=1}^n h_i^2\right) \\
  & \qquad  + \frac{1}{n}
    \left( L \tr(B^\top D) - h^\top
    (B\onevec - B^\top \onevec)
    \right) \\
\end{align*}
where $D \in \R^{n \times n}$ is a matrix with entries
$D_{ij} = \norm{x_i - x_j}+\norm{c_i-c_j}$. From strong duality, the primal
optimal value~\eqref{eq:primal-again} is
$\inf_{\gamma \in \R^n_+, \lambda \ge 0, B \in \R^{n \times n}_+} \sup_{h}
\mc{L}(h, \gamma, \lambda, B)$.

The first order conditions for the inner supremum give
\begin{align*}
  h\opt_i & \defeq \frac{1}{\lambda} \left(
  \lossitem-\eta + \gamma - (B\onevec - B^\top \onevec)_i
            \right). \\
\end{align*}
Substituting these values and
taking the infimum over $\lambda, \gamma \ge 0$,
we obtain
\begin{align*}
  \inf_{\lambda \ge 0, \gamma \in \R^n_+}
  \sup_{h} \mc{L}( h, \gamma, \lambda, B)
  & =  \bigg( \frac{1}{n} \sum_{i=1}^n
      \hingeBig{ \lossitem
      -  \sum_{j=1}^n (B_{ij} - B_{ji}) -\eta}^2
  \bigg)^{1/2} \\
&  + \frac{L}{n} \sum_{i,j = 1}^n (\norm{x_i - x_j} + \norm{c_i-c_j}) B_{ij}.
\end{align*}
Taking the infimum over $B,\eta $ and substituting this expression into the inner supremum of $R_L$ gives the desired estimator.

\section{Distortion Proof}
Terminology in this section generally follows that of the main text. We will use $c$ to describe some true set of unmeasured variables, and $\overline{c}$ to describe the elicited set. All notation with overhead lines are defined in this space of elicited unmeasured variables (e.g. $\ol{h}$, $\ol{\mc{H}_L}$).

Additionally we will define a forward map from true unmeasured variables to elicited ones, $f: \mc{C} \to \ol{\mc{C}}$ and a reverse map from elicited unmeasured variables to true ones $g: \ol{\mc{C}} \to \mc{C}$.

For convenience, define the following risk functionals for the DRO problem under the true unmeasured variables
\[R_L(\theta) := \inf_\eta \sup_{h\in\mc{H}_L} \frac{1}{\alpha}\EE_{x,y,c}[ h(x,c) \ell(x,y) - \eta] + \eta,\]
and under the estimated ones
\begin{equation}\label{eq:estrisk}
  \ol{R}_L(\theta) := \inf_\eta \sup_{\ol{h}\in\ol{\mc{H}_L}} \frac{1}{\alpha}\EE_{x,y,\ol{c}}[ \ol{h}(x,\ol{c}) \ell(x,y) - \eta] + \eta.
\end{equation}

We can define the upper bound for the Lipschitz case,
\begin{prop}\label{prop:errbound}
  Let $f: \mc{C} \to \ol{\mc{C}}$ define $\hat{h}(x,c) := \ol{h}(x,f(c))$ such that $\frac{1}{K_f} \hat{h} \in \mc{H}_L$ for all $\ol{h} \in \ol{\mc{H}_L}$. Then,
  \[  \ol{R}_L(\theta) \leq K_f R_L(\theta) + \frac{LM \EE_{xy} W_1(f(c|xy), \ol{c}|xy)}{\alpha}\]
  where $f(c|xy)$ is the pushforward measure of $c|xy$ under $f$.  
\end{prop}
\begin{proof}
  Let $\ol{h}^*$ be the $\ol{h}\in\ol{\mc{H}_L}$ which is the maximizer to Eq~\eqref{eq:estrisk}. For convenience define
  \begin{align*}
    \Delta_{xy}^f &:= E_{c|xy}[\hat{h}^*(x,c)] - E_{\ol{c}|xy}[\ol{h}^*(x,\ol{c})]\\
    &= E_{\ol{c}\sim f(c|xy)}[\ol{h}^*(x,\ol{c})] - E_{\ol{c}|xy}[\ol{h}^*(x,\ol{c})]
  \end{align*}
  The equality follows the change of variables property of pushforward measures.
Now rewriting the risk measure in terms of $\Delta$,
  \begin{align*}
    \ol{R}_L(\theta) &= \inf_\eta\frac{1}{\alpha} \EE_{xy}\left[\left(\E_{c|xy}[\hat{h}^*(x,c)]-\Delta_{xy}^f\right)\ell(x,y) - \eta\right]+\eta\\
                            &\leq \inf_\eta\frac{1}{\alpha} \EE_{xy} \left[E_{c|xy}[\hat{h}^*(x,c)]\ell(x,y)-\eta\right]+\eta \\
    &\qquad + \frac{E_{xy}[\left| \Delta_{xy}^f \right|] M}{\alpha}\\
                    &\leq \inf_\eta K_f \sup_{h\in \mc{H}_L} \frac{1}{\alpha} \EE_{xy} \left[\EE_{c|xy}[h(x,c)]\ell(x,y) - \eta\right]+\eta \\ &\qquad +\frac{E_{xy}[\left| \Delta_{xy}^f \right|] M}{\alpha}\\
                            &= K_fR_L(\theta) + \frac{E_{xy}[\left| \Delta_{xy}^f \right|] M}{\alpha}\\
                            &\leq K_fR_L(\theta) + \frac{L M W_1(f(c|xy), \ol{c}|xy)}{\alpha}
  \end{align*}
  First inequality follows from H\"older's inequality, and the fact that $0\leq \ell(x,y)\leq M$. The second one follows from the assertion that $\frac{1}{K_f} \hat{h} \in \mc{H}_L$, and the last inequality follows from the fact that $\ol{h}$ is $L$-Lipschitz, and utilizing the pushforward measure form of $\Delta$. 
\end{proof}

An analogous argument shows the other side of this bound given by,
\[  R_L(\theta) \leq K_g \ol{R}_L(\theta) + \frac{LM \EE_{XY} W_1(c|xy, g(\ol{c}|xy))}{\alpha}.\]
This shows that our DRO estimator achieves multiplicative error scaling with $K_f, K_g$ and additive error scaling with the Wasserstein distance between the true and the estimated unmeasured variables. 

Our assumptions on $K_f$ and $K_g$ are easily fulfilled in the case where there is a single bi-Lipschitz bijection $f: \mc{C}\to \ol{\mc{C}}$. In this case, $g=f^{-1}$ and $K_f = K_g = K$. 

We can interpret this bound as capturing two sources of error: our metric can be inappropriate and our estimates of $\ol{C}$ can be inherently noisy. For the first term, note that a map with higher metric distortion (e.g. bi-Lipschitz maps with large constants) results in a looser bound. This is because the Lipshcitz function assumption in the original space $\mc{C}$ does not correspond closely to Lipschitz functions in $\ol{\mc{C}}$.

For the second term, we incur error whenever $W_1(c|xy, g(\ol{c}|xy))$ is large. The alignment map $g$ takes our elicited unmeasured variables and approximates the true ones. However, if $\ol{c}$ does not contain enough information to reconstruct $c$ then no function $g$ can exactly map $\ol{c}$ to $c$, and we incur an approximation error that scales as the transport distance between the two.

We can now provide a simple lemma that bounds the quality of the model estimate under the approximation $\ol{c}$ compared to the minimizer of the exact unmeasured variables $c$.

For convenience we will use the following shorthand for the additive error terms,
\begin{align*}
  A_f &= \frac{LM \EE_{XY} W_1(\ol{c}|xy, f(c|xy))}{\alpha}\\
  A_g &=\frac{LM \EE_{XY} W_1(c|xy, g(\ol{c}|xy))}{\alpha}.
  \end{align*}

\begin{cor}
  Let $\ol{\theta}^* := \arg\min_\theta \ol{R}_L(\theta)$, then
  \begin{align*}
    &R_L(\ol{\theta}^*) - \inf_\theta R_L(\theta) \\
    &\leq \inf_\theta R_L(\theta)\left(K_fK_g-1\right)+K_gA_f + A_g \\
    \end{align*}
  \end{cor}
  \begin{proof}
    By Proposition~\ref{prop:errbound}, we have both
    \begin{align*}
      \inf_\theta \ol{R}_L(\theta) &\leq \inf_\theta K_f R_L(\theta) + A_f\\
      R_L(\ol{\theta}^*) &\leq K_g \ol{R}_L(\ol{\theta}^*) + A_g.
    \end{align*}
    By definition of $\ol{\theta}^*$ as the minimizer of $\ol{R}_L$, we obtain
    \[R_L(\ol{\theta}^*) \leq K_fK_g \inf_\theta R_L(\theta) + K_gA_f + A_g\]
    which gives the stated result.
  \end{proof}

  The corollary shows that the best model under the estimated unmeasured variables $\ol{c}$ performs well under the true DRO risk measure $R_L$ as long as $K_fK_g\approx 1$ and $A_f,A_g$ are small. There are two sources of error: the metric distortion results in a relative error that scales as $K_fK_g$, and the noise in estimation $(A_f,A_g)$ results in additive error. The $K_gA_f$ scaling term arises from the fact that error is measured with respect to the metric over $c$, not over $\ol{c}$. 

Importantly, these bounds show that we need not directly estimate the true unmeasured variables $c$ using $\ol{c}$ - our estimated unmeasured variables can live in an entirely different space, and as long as there \emph{exists} some low-distortion alignment functions $f,g$ that align the two spaces, the implied risk functions are similar.

\section{Effect of Crowdsourcing Quality}
\begin{figure}[h]
\begin{subfigure}{.5\textwidth}
  \centering
  \includegraphics[scale=0.275]{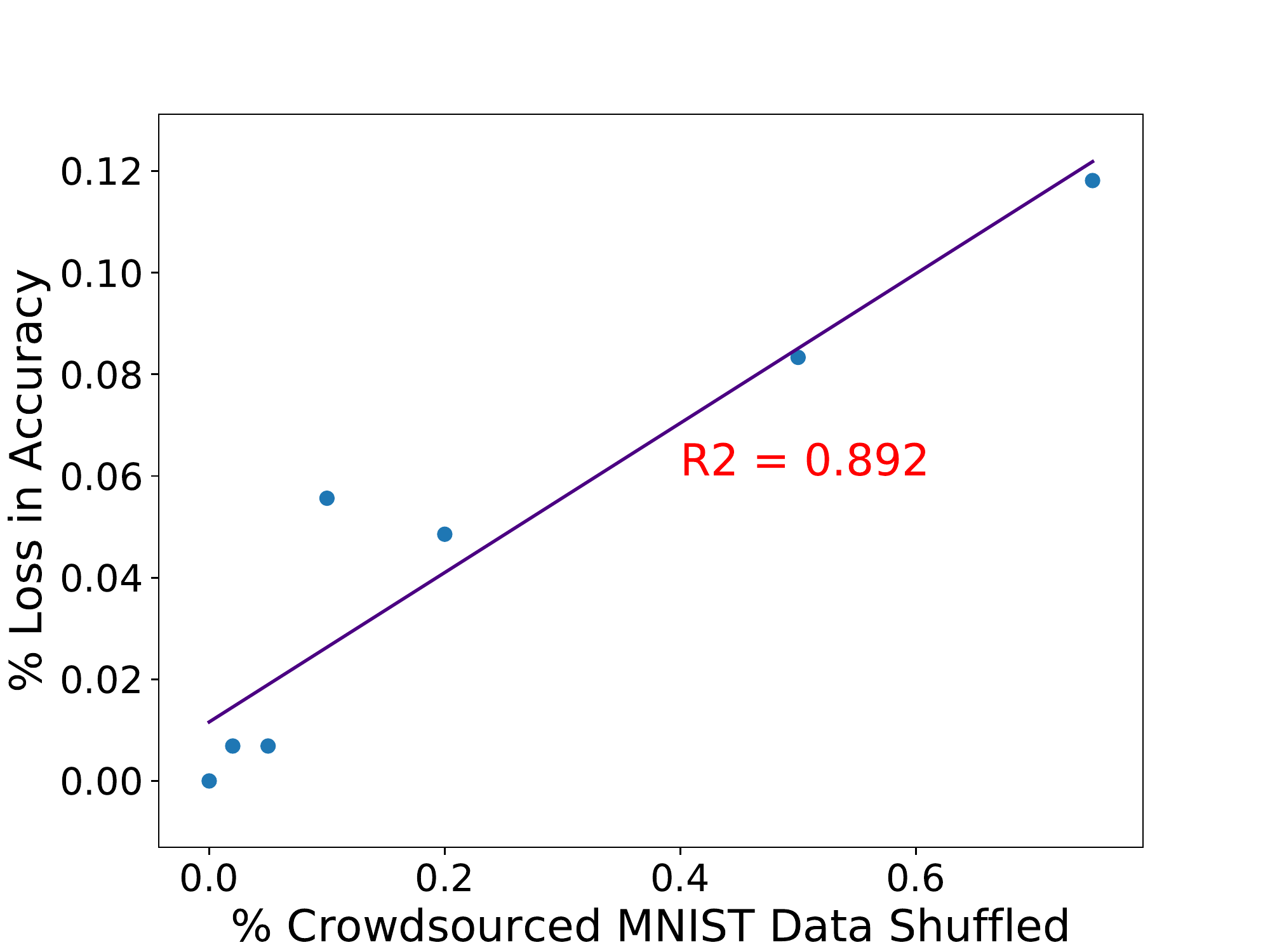}
  \end{subfigure}%
  \begin{subfigure}{.25\textwidth}
  \centering
  \includegraphics[scale=0.275]{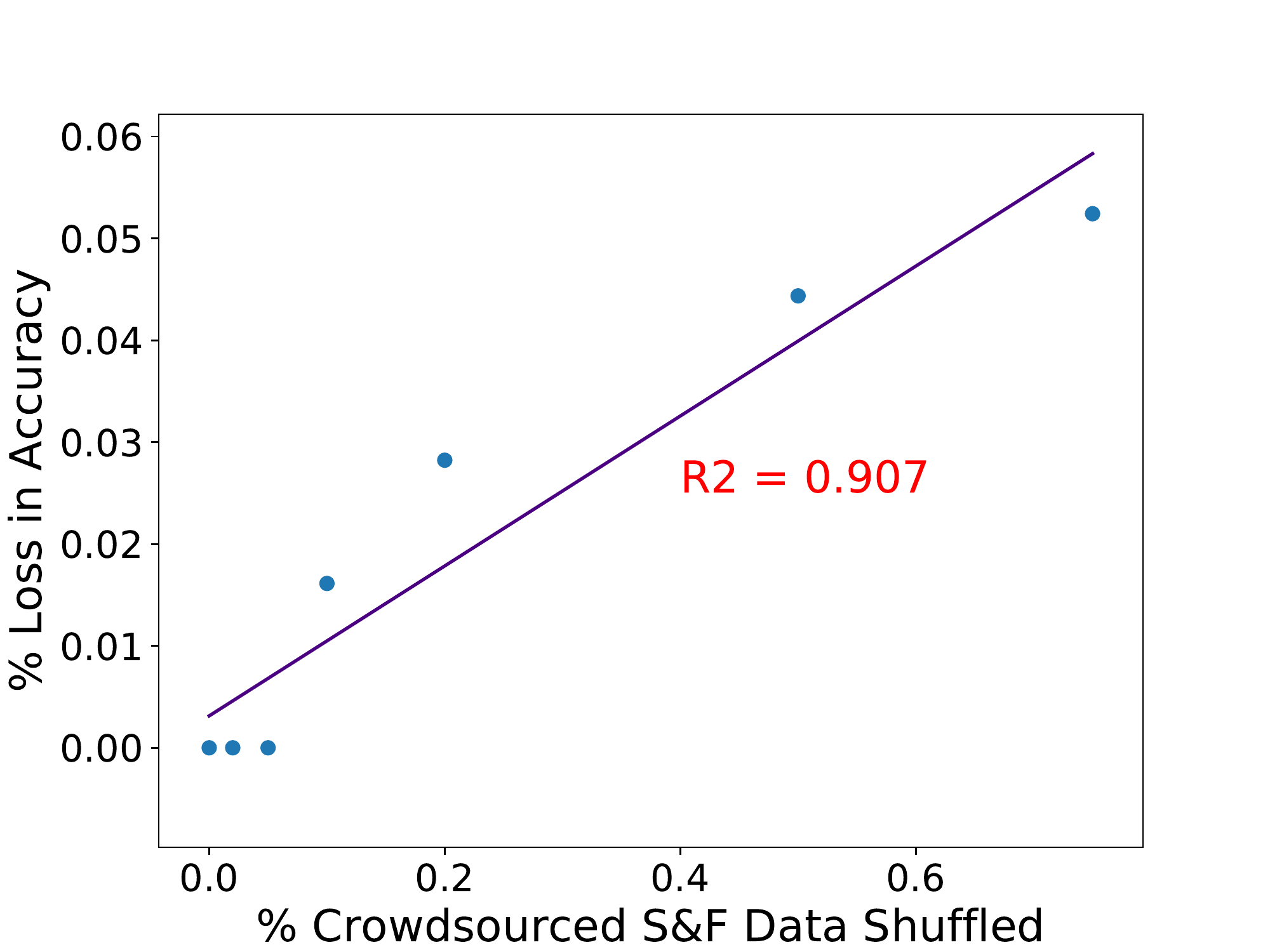}
  \end{subfigure}
  \caption{Decreasing crowdsourcing quality by randomly shuffling results in a highly correlated decrease in accuracy over both MNIST (left) and stop-and-frisk(right) datasets.}\label{fig:shuffle}
  \end{figure}
We empirically evaluate the role of crowdsourcing data quality on UV-DRO performance to complement our theoretical bound in Section 4. We previously showed a significant performance gap when we shuffle 100\% of the crowdsourced unmeasured variables, causing random associations that impact the crowdsourcing quality. We further investigate this gap by shuffling [0, 2, 5, 10, 20, 50, 75]\% of the crowdsourced unmeasured variables, and find a highly correlated accuracy drop for both MNIST ($R^2 = .89$) and stop-and-frisk datasets ($R^2 = .91$), as seen in Figure ~\ref{fig:shuffle}. This demonstrates a linear relationship between crowdsourcing quality and robust performance. 

\section{Annotation Unigrams Analysis Table}
\begin{table}[h]
\caption{Exploratory analysis on the annotations collected over stop-and-frisk data by training a logistic regression model to predict location from a selection of annotation unigrams.}
\label{sf-table}
\vskip 0.15in
\begin{center}
\begin{small}
\begin{sc}
\begin{tabular}{lcccr}
\toprule
\textbf{Brooklyn}&&\textbf{Manhattan}\\
\midrule
\textbf{Unigram}&\textbf{Weight}&\textbf{Unigram}&\textbf{Weight} \\
Discrimination& -1.22& Weapon& 0.82 \\
Racist & -0.29& Gun & 0.21\\
Racial    & -0.19& Armed & 0.89 \\
Homeless    & -0.84& Drug & 0.43        \\
Unrelated     &-1.68& Gang & 1.03\\
Cleared      &-0.98& Dangerous & 0.79 \\
Evidence      &-0.12& Witness & 0.81\\
\bottomrule
\end{tabular}
\end{sc}
\end{small}
\end{center}
\vskip -0.1in
\end{table}

\section{Reproducibility \& Experiment Details}
All experiments and data described below are available on CodaLab:
 \url{https://bit.ly/uvdro-codalab}.  

\subsection{Simulated Medical Diagnosis Task}
We simulate our data (n=1,000) using the following generation procedure:
\begin{enumerate}
\item $q_{train}$ = {.05, .1, .2, .3, .4, .5, .6, .7, .8} and $q_{test}$ = 0.8. 
\item $\conf$ is sampled from the $c \sim 1 - 2$ Bernoulli$(q)$. 
\item $y$ is sampled from $y \sim \mathcal{N}(0, 2)$, independent from from train or test.
\item For each $(c, y)$ sample, set $x_1 = \conf * y$ and $x_2 = y + \epsilon$ where $\epsilon \sim \mathcal{N}(0,4)$. 
\end{enumerate}  

For both ERM and UV-DRO, we trained a linear regression model over $p(y|x_1, x_2)$, optimized using batch gradient descent over 3k steps with AdaGrad with an optimal learning rate of $.0001$. We set UV-DRO parameter $\alpha = 0.2$, and tune $\eta$ via grid-search for each $q_{train}$ value. We present results (Mean Squared Error) on the same held-out test set for all models.

\subsection{MNIST Digit Classification with Confounding Transformations}
We use the popular MNIST dataset (\url{http://yann.lecun.com/exdb/mnist/}). We train on only a subset (n=4000) of the training data due to the cost of collecting annotations, and tune parameters on a separate validation set. For all data points, we treat the pixels of a (possibly transformed) image as the features $x$, the fact of whether a transformation occurred as the unmeasured variable $c$, and the MNIST digit as label $y$. We simulate a shift in an unmeasured rotation confounding variable using the following procedure:

\begin{enumerate}
\item $q_{train}$ = {.05, .1, .2, .4, .6} and $q_{test}$ = 1.0. 
\item $\conf$ is sampled from the $c \sim $ Bernoulli$(q)$, where $\conf = 1$ means the image was rotated. 
\item For each $(x,y)$ pair in the dataset, we rotate the original MNIST image $x$ by 180 degrees if $\conf = 1$.

\end{enumerate} 

For all ERM, DRO, and UV-DRO models, we trained a logistic regression model, optimized with batch gradient descent using AdaGrad and an optimal learning rate of $.001$. The optimal $l2$ penalty found for ERM models was 25. Optimal UV-DRO parameters (tuned on 20\% of data as valid) include $l2$ penalty of 50, a Lipschitz constant $L$ of 1, $\alpha = 0.2$, and we explicitly solve for the minimizer of $\eta$ with regards to the empirical distribution at each gradient step. We present results (Log-Loss, Accuracy) on the same held-out test set for all models.   

\subsection{Police Stop Analysis with Confounding Locations}
We use a dataset of NYPD police stops (\url{https://www.nyclu.org/en/stop-and-frisk-data}). We train on only a subset (n=2000) of the training data due to the cost of collecting annotations, and tune parameters on a separate validation set. For all data points, we filter out all variables except for 26 police stop observation as features $x$ (i.e. "in a high crime area"), the NYC borough as the unmeasured location variable $c$, and the label for arrest $y$. We simulate a shift in the location variable ($c$) using the following procedure:

\begin{enumerate}
\item $q_{train}$ = {0.1, 0.2, 0.3, 0.4, 0.5, 0.6} and $q_{test}$ = 1.0. 
\item $\conf$ is sampled from the $c \sim $ Bernoulli$(q)$, where $\conf = 1$ means the location is Brooklyn. 
\item We build the dataset by drawing from the entire dataset a $(x,y,c=\conf')$ example for each $\conf'$ sampled. 
\end{enumerate} 

For all ERM, DRO, and UV-DRO models, we trained a logistic regression model optimized with batch gradient descent using AdaGrad and an optimal learning rate of $.005$. The optimal $l2$ penalty found for ERM models was 0. Optimal UV-DRO (tuned on 20\% of data as valid) parameters include $l2$ penalty of 50, a Lipschitz constant $L$ of 1, $\alpha = 0.2$, and we explicitly solve for the minimizer of $\eta$ with regards to the empirical distribution at each gradient step. We present results (Log-Loss, Accuracy) on the same held-out test set for all models.    

\end{document}